%% file: ocd-iclr.tex
\documentclass{article} 
\usepackage[table,xcdraw]{xcolor}
\usepackage{iclr2019_conference,times}

\input{defs}

\usepackage[utf8]{inputenc} 
\usepackage[T1]{fontenc}    
\usepackage{hyperref}       
\usepackage{url}            
\usepackage{booktabs}       
\usepackage{amsfonts}       
\usepackage{nicefrac}       
\usepackage{microtype}      
\usepackage{enumitem}
\usepackage{soul}

\usepackage[colorinlistoftodos]{todonotes} 

\iclrfinalcopy
\setlist{nolistsep}

\title{Optimal Completion Distillation\\
for Sequence Learning\vspace*{-.3cm}}

\comment{
\author{
\begin{tabular}{c@{\hspace*{.5cm}}c@{\hspace*{.5cm}}c}
Sara Sabour & William Chan & Mohammad Norouzi\\
\end{tabular}\\[.1cm]
\texttt{\{sasabour,\:williamchan,\:mnorouzi\}@google.com}\\
Google Brain
}}
\author{%
Sara Sabour,~William Chan,~Mohammad Norouzi  \\
\texttt{\!\{sasabour,\,williamchan,\,mnorouzi\}@google.com} \\
Google Brain\\
}

\pgfplotsset{compat=1.14}
\begin{document}

\maketitle

\vspace*{-.2cm}
\begin{abstract}
\vspace*{-.2cm}
\input{abs}
\vspace*{-.2cm}
\end{abstract}

\vspace*{-.2cm}
\section{Introduction}
\vspace*{-.1cm}
\label{sec:intro}
\input{intro}

\vspace*{-.2cm}
\section{Background: Sequence Learning with MLE}
\vspace*{-.1cm}
\input{formulation}

\vspace*{-.2cm}
\section{Optimal Completion Distillation}
\vspace*{-.1cm}
\input{method}

\vspace{-.3cm}
\section{Related Work}
\vspace{-.2cm}
\input{relatedwork}

\input{experiments}

\input{conclusion}

\ifarxiv 
 \input{acknowledge.tex}
 \else
 \ificlrfinal
 \input{acknowledge.tex}
\fi
\fi
\bibliographystyle{plainnat}
\bibliography{bib}

\begin{appendices}

\input{supp.tex}
\end{appendices}

\end{document}

%% file: defs.tex
\usepackage{xspace}
\def\acronym{OCD\xspace}

\usepackage{caption}        
\usepackage{subcaption}        
\usepackage{amsthm}
\usepackage{bbm}
\usepackage{amsmath}
\usepackage{mathtools}

\usepackage{pgfplots, pgfplotstable}
\usepackage{algorithm}

\usepackage{enumitem}
\usepackage{multirow}
\usepackage[title]{appendix}

\usepackage{color}

\usepackage[noend]{algpseudocode} 
\newcommand{\mohammad}[1]{\textcolor{magenta}{(Mo: #1)}}

\newcommand{\rebut}[1]{\textcolor{black}{#1}}
\newcommand{\comment}[1]{}
\newcommand{\new}[1]{\textcolor{blue}{#1}}

\floatname{algorithm}{Procedure}

\renewcommand{\citeauthoryear}[1]{(\citeauthor{#1}; \citeyear{#1})}

\newcommand{\mmin}[1]{\underset{#1}{\operatorname{min}}}
\newcommand{\mmax}[1]{\underset{#1}{\operatorname{max}}}
\newcommand{\argmin}[1]{\underset{#1}{\operatorname{argmin}}}

\newcommand{\one}[1]{\mathbbm{1}[#1]}

\def\libclean{4.5}
\def\libother{13.3}

\definecolor{Fcolor}{HTML}{af2418}

\definecolor{Ccolor}{HTML}{ffd359}

\renewcommand{\vec}[1]{\boldsymbol{\mathbf{#1}}}

\def\a{y}
\def\ba{\vec{y}}
\def\bas{\vec{y}^*}
\def\bai{\hat{\ba}}
\def\ed{\mathrm{D}_\text{edit}}
\def\pit{p_{\theta}}
\def\pitt{p_{\theta,t}}

\def\bs{\vec{s}}

\def\bu{\vec{u}}
\def\bv{\vec{v}}

\def\bx{\vec{x}}
\def\by{\vec{y}}
\def\byt{\tilde{\vec{y}}}

\def\yt{\tilde{y}}

\def\ra{{\textnormal{a}}}

\def\bys{{\mathbf y}^*}
\def\ys{{y}^*}
\def\byt{\widetilde{\mathbf y}}

\def\V{\mathcal{V}}
\def\Y{\mathcal{Y}}

\def\int{\mathrm{int}}
\DeclareMathOperator{\E}{\mathbb{E}}

\newcommand{\kl}[2]{\mathrm{KL}\left(#1~\Vert~#2\right)}

\def\eg{{\em e.g.}}
\def\ie{{\em i.e.}}
\def\vs{{\em v.s.}}
\def\aka{{\em a.k.a.}}

\newtheorem{theorem}{Theorem}
\newtheorem{lemma}[theorem]{Lemma}

\newcommand{\tabref}[1]{Table~\ref{#1}}
\newcommand{\figref}[1]{Figure~\ref{#1}}
\newcommand{\secref}[1]{Section~\ref{#1}}

%% file: abs.tex
We present {\em Optimal Completion Distillation} (OCD), a training procedure for optimizing
sequence to sequence models based on edit distance. OCD is efficient, has no hyper-parameters
of its own, and does not require pretraining or joint optimization with conditional log-likelihood.
Given a partial sequence generated by the model, we first identify the set of optimal suffixes
that minimize the total edit distance, using an efficient dynamic programming algorithm.
 Then, for each position of the generated sequence, we define a target distribution that puts an equal
 probability on the first token of each optimal suffix. OCD achieves the state-of-the-art
 performance on end-to-end speech recognition, on both Wall Street Journal and Librispeech datasets,
 achieving $9.3\%$ and \rebut{$\libclean\%$} word error rates, respectively.

\comment{Then, at each position of the generated sequence, we encourage the model to continue with the set of optimal extensions.}
\comment{we develop a dynamic programming algorithm to exactly 
At each position in the sequence, we teach the model to mimic the optimal suffixes.
}

\comment{,
outperforming all previous techniques, including our own well tuned teacher forcing and scheduled sampling baselines.}

\comment{Neural sequence models have revolutionized multiple application domains.
These models have been predominantly trained via teacher forcing and conditional log-likelihood.}

\comment{

Neural sequence models have revolutionized multiple application domains.
These models have been predominantly trained via teacher forcing and conditional log-likelihood.
We present Optimal Completion Distillation (OCD), a training procedure that does not require initialization nor joint optimization with maximum likelihood.
Given any arbitrary prefix, (e.g. a partial sequence generated by sampling from the model), we develop an efficient dynamic programming algorithm to compute the optimal actions that minimizes the total edit distance.
At each position in the sequence, we distill the optimal completion policy to the model.
We achieve the state-of-the-art performance on the Wall Street Journal speech recognition task, achieving 3.1\% CER and 9.3\% WER,
outperforming all other previous work, including our own well tuned teacher forcing and scheduled sampling baselines.

Neural sequence models have revolutionized multiple application domains.
These models have been predominantly trained via teacher forcing and conditional log-likelihood.
We present Optimal Completion Distillation (OCD), a training procedure that does not require initialization nor joint optimization with conditional likelihood.
Given any arbitrary prefix, (e.g. a partial sequence generated by sampling from the model), we develop an efficient dynamic programming algorithm to compute the optimal tokens that minimizes the total edit distance.
At each position in the sequence, we distill the optimal completion policy into the model.
We achieve the state-of-the-art performance on the Wall Street Journal speech recognition task, achieving 3.1\% CER and 9.3\% WER,
outperforming all other previous work, including our own well tuned teacher forcing and scheduled sampling baselines.

}

%% file: intro.tex
Recent advances in natural language processing and speech recognition
hinge on the development of expressive neural network architectures for
sequence to sequence (seq2seq)
learning~\citep{sutskeveretal14,bahdanau2014neural}. Such
encoder-decoder architectures are adopted in both machine translation~\citep{bahdanau2014neural,wu2016google,parity2018microsoft} and speech recognition systems~\citep{chan-icassp-2016,bahdanau-icassp-2016,chiu2017state} achieving impressive performance above traditional multi-stage pipelines~\citep{koehn2007moses,povey2011kaldi}. 
Improving the building blocks of seq2seq models can fundamentally advance machine translation and speech
recognition, and positively impact other domains such as image captioning \citep{showattendtell2015},
parsing \citep{vinyals-nips-2015}, summarization \citep{rush-emnlp-2015}, and program
synthesis \citep{seq2sql}.

To improve the key components of seq2seq models, one can either design better architectures,
or develop better learning algorithms. Recent
architectures using convolution~\citep{gehring-arxiv-2017} and self
attention~\citep{transformer2017} have proved to be useful, especially
to facilitate efficient training. On the other hand, despite many attempts to
mitigate the limitations of Maximum Likelihood Estimation (MLE)
\citep{ranzato-iclr-2016,wiseman-emnlp-2016,norouzi-nips-2016,bahdanau-iclr-2017,searnn2017},
MLE is still considered the dominant approach for training seq2seq
models\comment{ \new{from scratch} \mohammad{it is also true more generally, no?}}. Current alternative approaches require pre-training or 
joint optimization with conditional log-likelihood. They are difficult to implement and
require careful tuning of new hyper-parameters (\eg~mixing ratios). In addition, alternative
approaches typically do not offer a substantial performance improvement over a well tuned MLE
baseline, especially when label smoothing~\citep{pereyra-iclr-2017,edunov2017classical} and
scheduled sampling~\citep{bengio-nips-2015} are used.


In this paper, we borrow ideas from search-based structured prediction~\citep{daumeetal09,ross-aistats-2011}
and policy distillation~\citep{rusu-iclr-2016} and develop an efficient algorithm for optimizing seq2seq models
based on edit distance\footnote{Edit distance between two sequences $\bu$ and $\bv$ is the minimum number of insertion,
deletion, and substitution edits required to convert $\bu$ to $\bv$ and {\em vice versa}.}. Our key observation is that
given an arbitrary prefix (\eg~a partial sequence generated by sampling from the model), we can {\em exactly} and {\em efficiently}
identify all of the suffixes that result in a minimum total edit distance (\vs~the ground truth target). Our training procedure,
called {\em Optimal Completion Distillation (OCD)}, is summarized as follows:
\begin{enumerate}[noitemsep,parsep=0pt,leftmargin=7mm]
\comment{    \item We always train on output sequences generated by sampling from the seq2seq model.}
    \item We always train on prefixes generated by sampling from the model that is being optimized.
    \item For each generated prefix, we identify all of the optimal suffixes that
    result in a minimum total edit distance \vs~the ground truth target using an efficient dynamic programming algorithm.
    \item We teach the model to {\em optimally extend} each generated prefix
    by maximizing the average log probability of the first token of each optimal suffix identified in step 2.
\end{enumerate}
\comment{In sum, during training, we distill the knowledge of the optimal completions into a seq2seq model.
During inference, we resort to the standard beam search process.} The proposed OCD algorithm is efficient,
straightforward to implement, and has no tunable hyper-parameters of its own.
\comment{Further, OCD provides a significant performance gain over a well tuned MLE baseline on two highly competitive
speech recognition benchmarks.}Our key contributions include:
\begin{itemize}[noitemsep,parsep=0pt,leftmargin=7mm]
    \item
    We propose \acronym, a stand-alone algorithm for optimizing seq2seq models based on
    edit distance. \acronym is scalable to real-world datasets with long sequences and large vocabularies, and consistently
    outperforms Maximum Likelihood Estimation (MLE) \rebut{by a large margin}. 
    \item Given a target sequence of length $m$ and a generated sequence of length $n$, we present an $O(nm)$ algorithm
    that identifies all of the optimal extensions for each prefix of the generated sequence.
    \item We demonstrate the effectiveness of \acronym on end-to-end speech
    recognition using attention-based seq2seq models. On the Wall Street Journal
    dataset, \acronym achieves a Character Error Rate (CER) of $3.1\%$ and a Word Error Rate (WER) of
    $9.3\%$ without language model rescoring, outperforming all prior work \rebut{(\tabref{tab:results})}. 
    On Librispeech, \acronym achieves state-of-the-art WER of \rebut{$\libclean\%$ on ``test-clean'' and $\libother\%$ on ``test-other'' sets (\tabref{tab:lib})}.
\end{itemize}

\comment{

These models are often trained to maximize the log-probability of the
correct output sequence using tractable autoregressive models. Such
models generate the output sequence one token at a time, typically in
a left-to-right fashion.

{\em Teacher forcing}~\citep{teacherforcing1989} for the optimization
of the{\em~conditional log-likelihood} objective is the standard
training procedure. A teacher provides a student with a{\em~prefix}
of the correct output sequence, $(y^*_1, \ldots, y^*_{i-1})$, and the
student optimizes the log-probability of the next correct token,
$y^*_{i}$. This resembles a teacher walking a student through a
sequence of perfect decisions, where the student learns as a passive
observer. However, during inference, the student needs to act
autonomously. To generate a token, $\hat{y}_i$, the student needs to
condition on their own previous outputs,
$(\hat{y}_1, \ldots, \hat{y}_{i-1})$, since the correct outputs are
not available anymore.

We highlight two problems with teacher forcing and Maximum Likelihood
Estimation (MLE):
\begin{enumerate}
    \item There is a mismatch between the prefixes seen by the model
    during training and inference. If the student's predictions
    deviate from the correct outputs, \ie~when the distribution of
    $(\hat{y}_1, \ldots, \hat{y}_{i-1})$ is different from the
    distribution of $(y^*_1, \ldots, y^*_{i-1})$, then the student may
    find themselves in a novel situation that they have not been
    trained for. The mismatch between the training and test
    distributions can result in poor generalization, especially when
    the training set is small or the model size is large.

    \item There is a mismatch between the training loss and the task
    objective. During training, one optimizes the log-probability of
    the correct output sequence, however this is often different from
    the task evaluation metric (\eg~edit distance or BLEU).
\end{enumerate}

There has been a recent surge of interest in understanding and
mitigating the limitations of teacher forcing and the log-likelihood
objective
(\eg~\cite{bengio-nips-2015,ranzato-iclr-2016,wiseman-emnlp-2016,norouzi-nips-2016,bahdanau-iclr-2017}). We
refer the reader to \secref{sec:relatedwork} for a detailed discussion
of prior work. We present a novel solution to both of the problems
discussed above based on reinforcement learning techniques, while
maintaining the efficiency and effectiveness of teacher forcing. Our
algorithm is easy to implement, and unlike most alternative
approaches, does not require initialization nor joint training with
MLE.

Our key intuition is that a competent teacher should not train a
student only on correct prefixes, but should also teach the student
about the best set of next tokens for an incorrect prefix. This is
especially helpful when the prefix includes some of the student's
mistakes. We generate a full sequence from the student, and at each
position, we train the student to predict the next set of tokens that
would result in an optimal completion. Our key technical contribution
is the development of a dynamic programming algorithm to find the
optimal completion for any prefix to minimize the total token error
rate. Our training procedure, {\em Optimal Completion Distillation}
(\acronym), distills the knowledge of the optimal completions into the
student.

We demonstrate the effectiveness of \acronym on end-to-end speech
recognition using attention-based sequence-to-sequence models.  We
adopt the challenging Wall Street Journal dataset and compare \acronym
against various previous techniques, including our well-tuned
implementation of teacher forcing and scheduled sampling. \acronym
achieves the state-of-the-art end-to-end speech recognition
performance with $3.1\%$ CER and $9.4\%$ WER, outperforming all prior
work
including \citep{bahdanau-icassp-2016,chorowski-interspeech-2017,chan-iclr-2017}.

}

\comment{
Neural sequence
models \citep{sutskeveretal14,bahdanau2014neural,gehring-arxiv-2017}
have seen remarkable success across many application domains including
machine translation \citep{wu2016google,parity2018microsoft},
speech recognition \citep{deepspeech2016,chan-icassp-2016}, and image
captioning \citep{showtell2015,showattendtell2015}.
These models are often trained to maximize the log-probability of the
correct output sequence using tractable autoregressive models. Such
models generate the output sequence one token at a time, typically in
a left-to-right fashion.

{\em Teacher forcing}~\citep{teacherforcing1989} for the optimization
of the {\em conditional log-likelihood} objective is the standard
training procedure. A teacher provides a student with a {\em prefix}
of the correct output sequence, $(y^*_1, \ldots, y^*_{i-1})$, and the
student optimizes the log-probability of the next correct token,
$y^*_{i}$. This resembles a teacher walking a student through a
sequence of perfect decisions, where the student learns as a passive
observer. However, during inference, the student needs to act
autonomously. To generate a token, $\hat{y}_i$, the student needs to
condition on their own previous outputs,
$(\hat{y}_1, \ldots, \hat{y}_{i-1})$, since the correct outputs are
not available anymore.

We highlight two problems with teacher forcing and Maximum Likelihood Estimation (MLE):
\begin{enumerate}
    \item There is a mismatch between the prefixes seen by the model during training and inference. If the student's predictions deviate from the correct outputs, \ie~when the distribution of $(\hat{y}_1, \ldots, \hat{y}_{i-1})$ is different from the distribution of $(y^*_1, \ldots, y^*_{i-1})$, then the student may find themselves in a novel situation that they have not been trained for. The mismatch between the training and test distributions can result in poor generalization, especially when the training set is small or the model size is large.
    \item There is a mismatch between the training loss and the task objective. During training, one optimizes the log-probability of the correct output sequence, however this is often different from the task evaluation metric (\eg~edit distance or BLEU).
\end{enumerate}

There has been a recent surge of interest in understanding and mitigating the limitations of teacher forcing and the log-likelihood objective \citep{bengio-nips-2015,ranzato-iclr-2016,wiseman-emnlp-2016,norouzi-nips-2016,bahdanau-iclr-2017}. We refer the reader to \secref{sec:relatedwork} for a detailed discussion of prior work. We present a novel solution to both of the problems discussed above based on reinforcement learning techniques, while maintaining the efficiency and effectiveness of teacher forcing. Our algorithm is easy to implement, and unlike most alternative approaches, does not require initialization nor joint training with MLE.

Our key intuition is that a competent teacher should not train a student only on correct prefixes, but should also teach the student about the best set of next tokens for an incorrect prefix. This is especially helpful when the prefix includes some of the student's mistakes. We generate a full sequence from the student, and at each position, we train the student to predict the next set of tokens that would result in an optimal completion. Our key technical contribution is the development of a dynamic programming algorithm to find the optimal completion for any prefix to minimize the total token error rate. Our training procedure, {\em Optimal Completion Distillation} (\acronym), distills the knowledge of the optimal completions into the student.

We demonstrate the effectiveness of \acronym on end-to-end speech recognition using attention-based sequence-to-sequence models.
We adopt the challenging Wall Street Journal dataset and compare \acronym against various previous techniques, including our well-tuned
implementation of teacher forcing and scheduled sampling. \acronym achieves the state-of-the-art end-to-end speech recognition performance with $3.1\%$ CER and $9.4\%$ WER, outperforming all prior work including \citep{bahdanau-icassp-2016,chorowski-interspeech-2017,chan-iclr-2017}.
}

%% file: formulation.tex
Given a dataset of input output pairs $\mathcal{D} \equiv \{(\bx, \bys)_i\}_{i=1}^N$,
we are interested in learning a mapping $\bx \to \by$ from an input $\bx$ to a target output sequence $\bys \in \mathcal{Y}$.
Let $\mathcal{Y}$ denote the set of all sequences of tokens from a finite vocabulary $\V$ with variable but finite lengths.
\comment{
$= (y^*_1, y^*_2, ...,y^*_T)$ of length $T$.
We assume that the tokens in $\bys$ are from a known finite vocabulary. Let $\by \in \mathcal{Y}$ denote a member of the set of all possible output sequences.}
Often learning a mapping $\bx \to \by$ is formulated as optimizing the parameters of a conditional distribution $\pit(\ba \mid \bx)$. Then, the final sequence prediction under the probabilistic model $\pit$ is performed by exact or approximate
inference (\eg~via beam search) as:
\begin{equation}
\bai ~\approx~ \mathrm{argmax\,}_{\ba \in \mathcal{Y}}\: \pit(\ba \mid \bx)~.
\label{eq:inference}
\end{equation}
Similar to the use of log loss for supervised classification, the standard approach to optimize the parameters $\theta$ of 
the conditional probabilistic model entails maximizing a conditional log-likelihood objective, $\mathcal{O}_{\text{MLE}}(\theta) = \E\nolimits_{(\bx, \bas) \sim p_\mathcal{D}} \log \pit(\bas \mid \bx)$.
This approach to learning the parameters is called Maximum Likelihood Estimation (MLE) \rebut{ and is commonly used in sequence to sequence learning.}

\cite{sutskeveretal14} propose the use of recurrent neural networks (RNNs)
for {\em autoregressive} seq2seq modeling to tractably optimize $\mathcal{O}_{\text{MLE}}(\theta)$.
An autoregressive model estimates the conditional probability of the target sequence given the source
one token at a time, often from left-to-right. A special {\em end-of-sequence} token is appended at
the end of all of target sequences to handle variable length.
The conditional probability of $\bys$ given $\bx$ is decomposed via the chain rule as,
\begin{equation}
\pit(\bas \mid \bx) ~\equiv~ \prod\nolimits_{t=1}^{\lvert \bas \rvert} \rebut{\pitt}(\a^*_t \mid \bas_{<t},\bx)~,
\label{eq:autoregressive}
\end{equation}
\comment{\coloneqq}
where $\bas_{<t} \equiv (\a^*_1, \ldots, \a^*_{t-1})$ denotes a prefix of
the sequence $\bas$. To estimate the probability of a token $a$ given a prefix $\bas_{<t}$ and an input $\bx$, denoted $\rebut{\pitt}(a \mid \bas_{<t},\bx)$, different architectures
have been proposed. 
Some papers~(\eg~\cite{britz2017massive}) have investigated the use of LSTM~\citep{hochreiter-nc-1997} and GRU~\citep{cho2014learning}
cells, while others proposed new architecturs based on soft attention~\citep{bahdanau2014neural}, convolution~\citep{gehring-arxiv-2017}, and self-attention~\citep{transformer2017}.
\rebut{Nonetheless, all of these techniques rely on MLE for learning,
\begin{equation}
\mathcal{O}_{\text{MLE}}(\theta) ~=~ \E\nolimits_{(\bx, \bas) \sim p_\mathcal{D}} \sum\nolimits_{t=1}^{\lvert \bas \rvert} \log \pitt(\a^*_t \mid \bas_{<t},\bx)~,
\label{eq:cll}
\end{equation}
where $p_\mathcal{D}$ denotes the empirical data
distribution, uniform across the dataset $\mathcal{D}$.
}
We present a new objective function for optimizing autoregressive seq2seq models applicable to any neural architecture.


\subsection{Limitations of MLE for Autoregressive Models}

In order to maximize the conditional log-likelihood \eqref{eq:cll} of an autoregressive seq2seq model \eqref{eq:autoregressive},
one provides the model with a {\em prefix} of $t-1$ tokens from the ground truth target sequence, denoted $\bys_{<t}$, and 
maximizes the log-probability of $\ys_t$ as the next token. This resembles a teacher walking a student through a
sequence of perfect decisions, where the student learns as a passive
observer. However, during inference one uses beam search \eqref{eq:inference}, wherein the student needs to
generate each token $\bai_t$ by conditioning on its own previous outputs,
\ie~$\bai_{<t}$ instead of $\bys_{<t}$. 
This creates a discrepancy between training and test known as {\em exposure bias} \citep{ranzato-iclr-2016}. Appendix~\ref{sec:exposure} expands this further.

Concretely, we highlight two limitations with the use of MLE for autoregressive seq2seq modeling:
\begin{enumerate}[noitemsep,parsep=0pt,leftmargin=7mm]
    \item There is a mismatch between the prefixes seen by the model during
    training and inference. When the distribution of
    $\bai_{<t}$ is different from the distribution of $\bys_{<t}$,
    then the student will find themselves in a novel situation that they have not been
    trained for. 
    This can result in poor generalization, especially when
    the training set is small or the model size is large.
    \item There is a mismatch between the training loss and the task
    evaluation metric. During training, one optimizes the log-probability of
    the ground truth output sequence, which is often different from
    the task evaluation metric (\eg~edit distance for speech recognition).
\end{enumerate}

There has been a recent surge of interest in understanding and
mitigating the limitations of MLE for autoregressive seq2seq modeling.
In \secref{sec:relatedwork} we discuss prior work in detail after presenting
our approach below.

\comment{We present a novel solution to both of the problems
discussed above based on reinforcement learning techniques, while
maintaining the efficiency and effectiveness of teacher forcing. Our
algorithm is easy to implement, it is stable and unlike most alternative
approaches, does not require pretraining nor joint training with
MLE for variance reduction.}

\comment{
\vspace*{-.2cm}
\subsection{Beyond Teacher Forcing and Conditional Log-Likelihood}
\vspace*{-.1cm}
The task loss is typically captured via a similarity function that compares $\bai$ and $\bas$,
denoted $R(\bas, \bai)$. We call this similarity function a reward function to highlight the connection between sequence prediction and
reinforcement learning. Common reward functions for sequence prediction include {\em negative edit distance} and {\em BLEU} score.
Once the model is trained, the quality of the model is measured in terms of its empirical performance on a held out dataset $\mathcal{D'}$ as,
\begin{equation}
  \mathcal{O}_{\text{ER}}(\theta) = \sum_{(\bx, \bas) \in \mathcal{D'}} R(\bas,\bai \approx \mathrm{argmax}_{\ba}\:\pit(\ba \mid \bx))~.
\label{eq:er}
\end{equation}
A key limitation of teacher forcing for sequence learning stems from the discrepancy between the training
and test objectives. One trains the model using conditional log-likelihood $\mathcal{O}_{\text{CLL}}$, but evaluates the quality of the model using
empirical reward $\mathcal{O}_{\text{ER}}$.

\begin{figure}[t]
\vspace*{-.2cm}
\begin{center}
\begin{tabular}{c@{\hspace*{2cm}}c}
      (a) {Teacher Forcing} &
      (b) {Scheduled Sampling} \\[-.1cm]
      \includegraphics[width=0.35\textwidth]{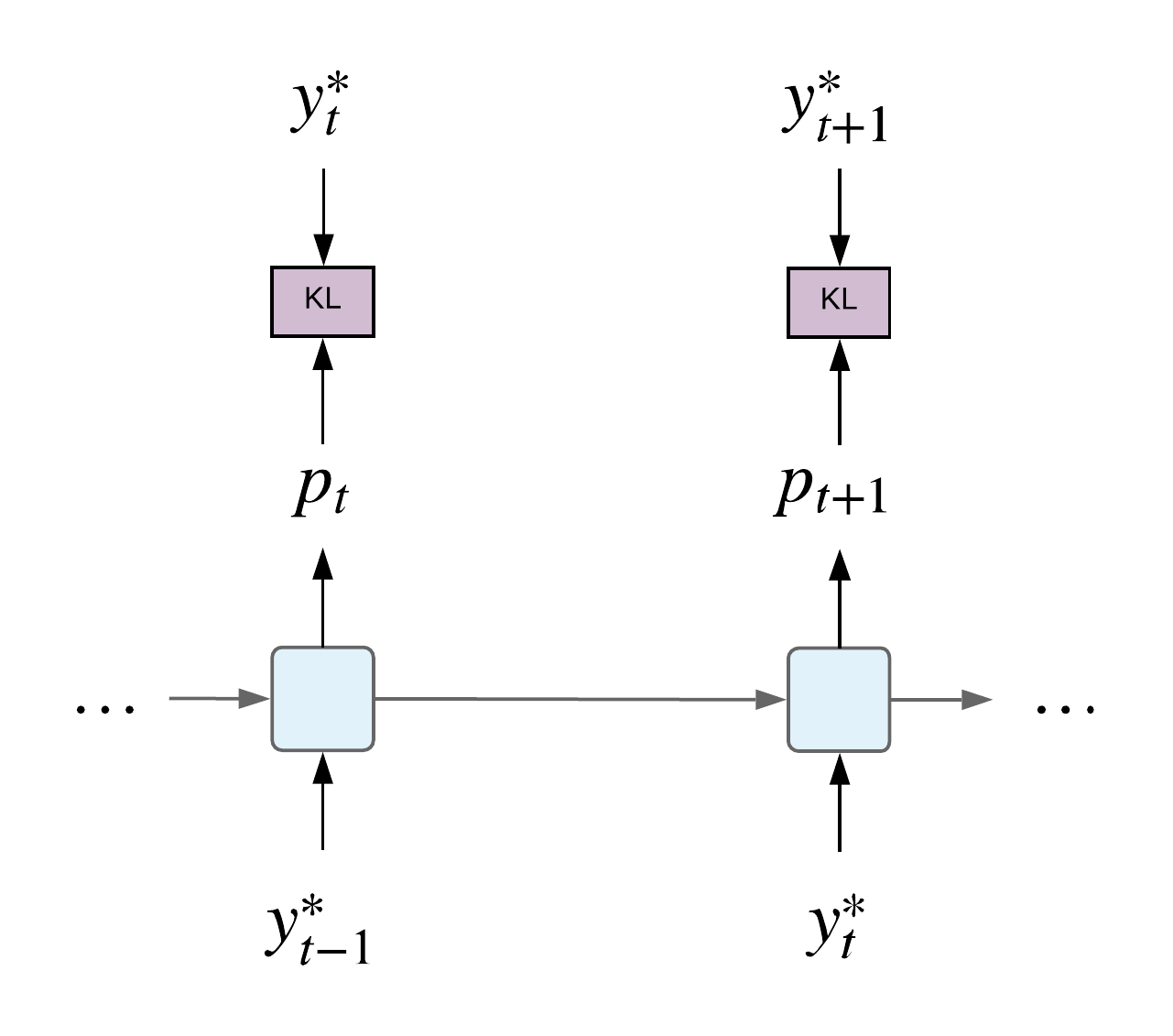} &
      \includegraphics[width=0.35\textwidth]{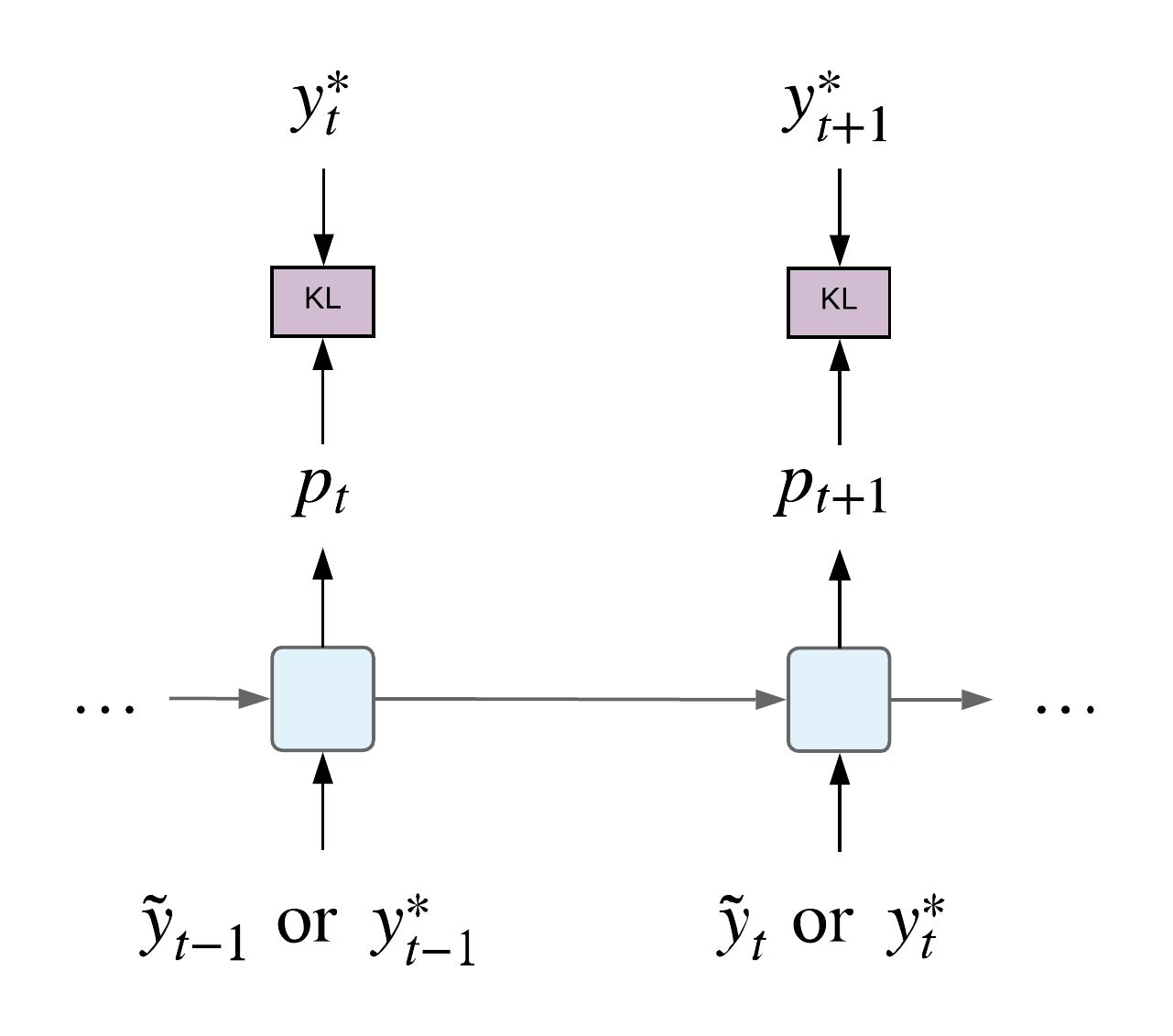} \\[-.1cm]
      \includegraphics[width=0.35\textwidth]{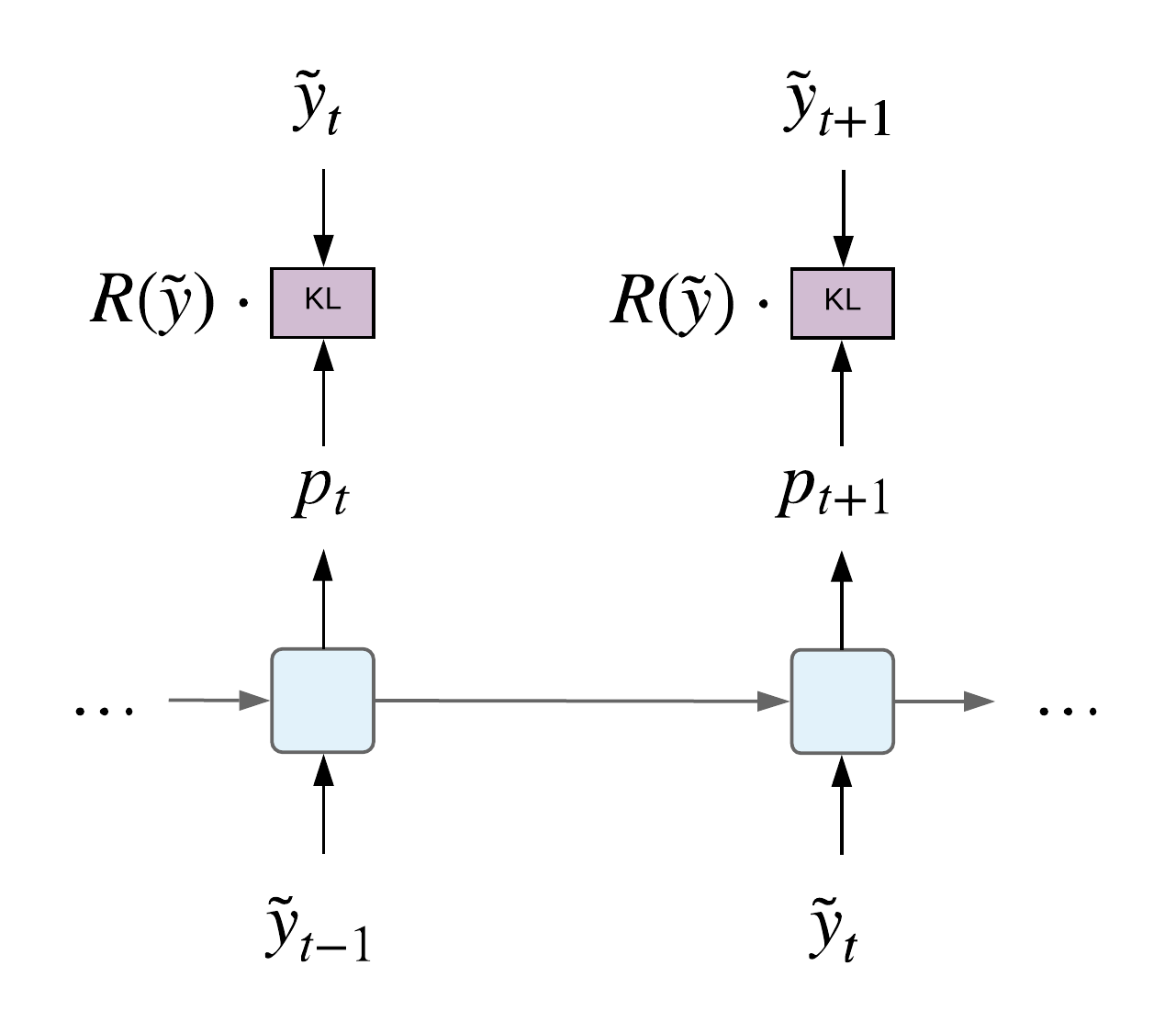} &
      \includegraphics[width=0.35\textwidth]{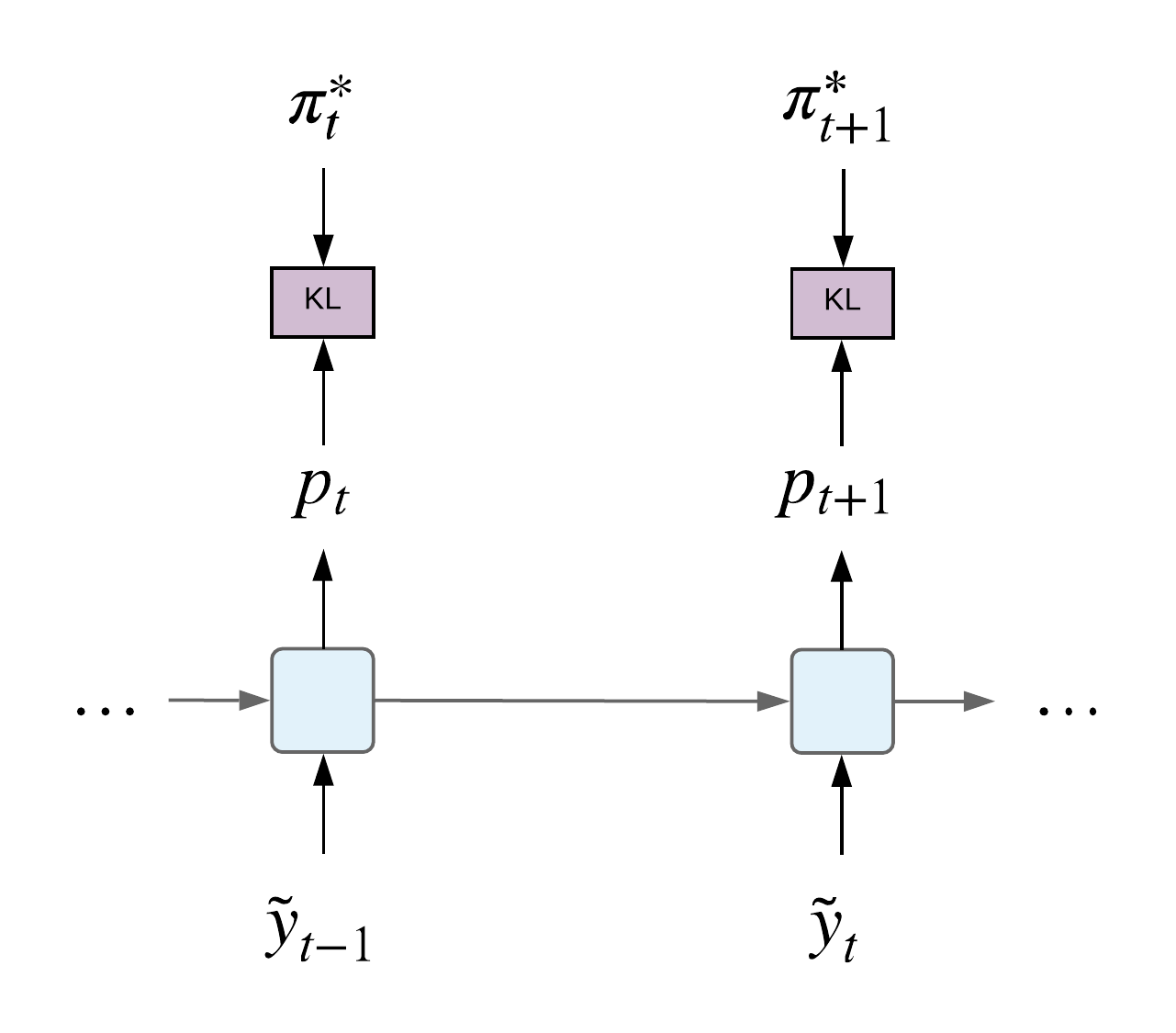} \\[-.1cm]
      (c) {Policy Gradient} &
      (d) {Optimal Completion Distillation}\\
\vspace*{-.2cm}
\end{tabular}
\end{center}
\comment{
                \end{tabular}
                \end{figure}
                \begin{figure}[t]
                  \centering
                  \begin{subfigure}[t]{0.5\textwidth}
                      \centering
                      \includegraphics[width=0.75\textwidth]{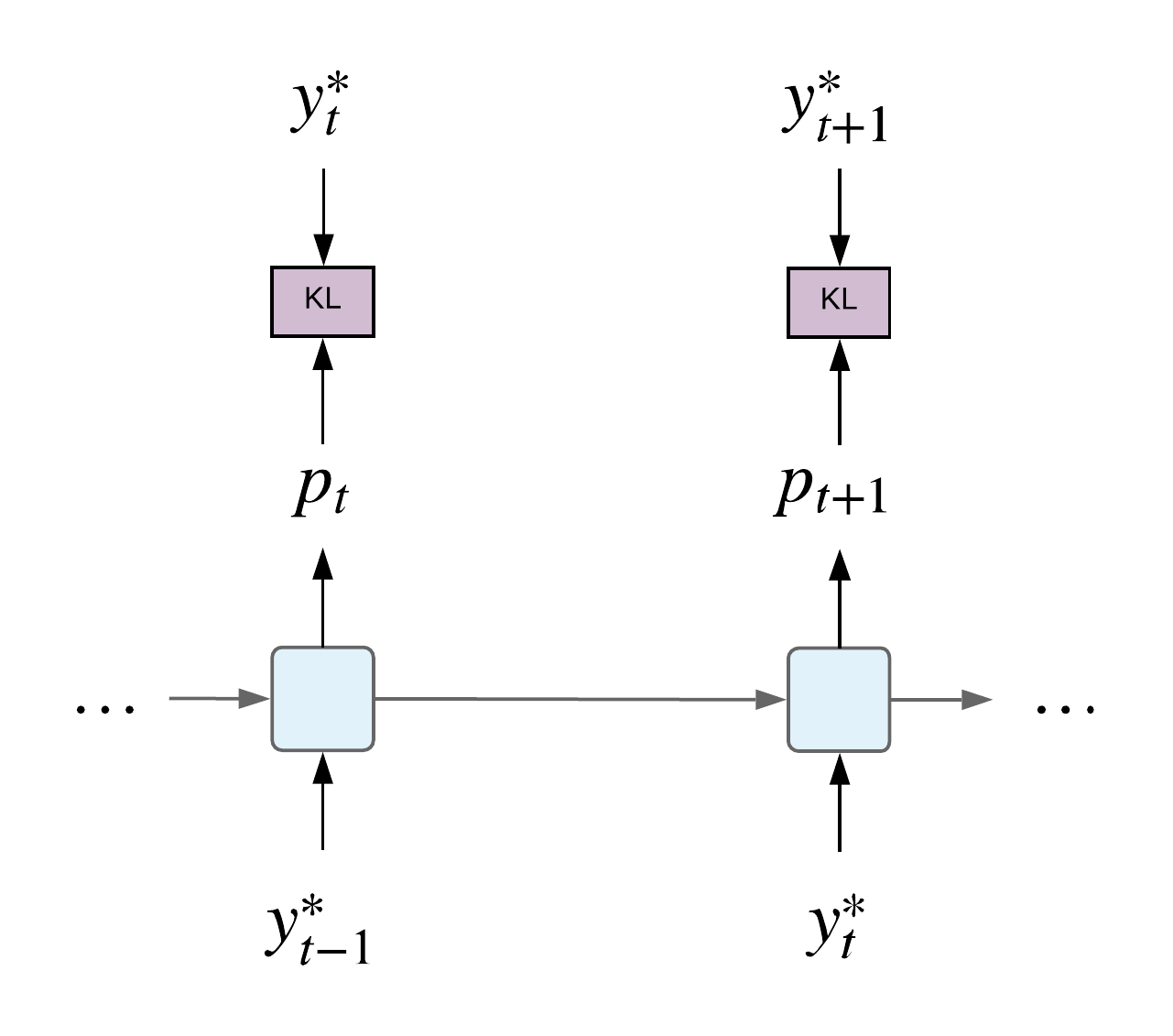}
                      \caption{Teacher Forcing}
                      \label{subfig:teacher-forcing}
                  \end{subfigure}%
                  \begin{subfigure}[t]{0.5\textwidth}
                      \centering
                      \includegraphics[width=0.75\textwidth]{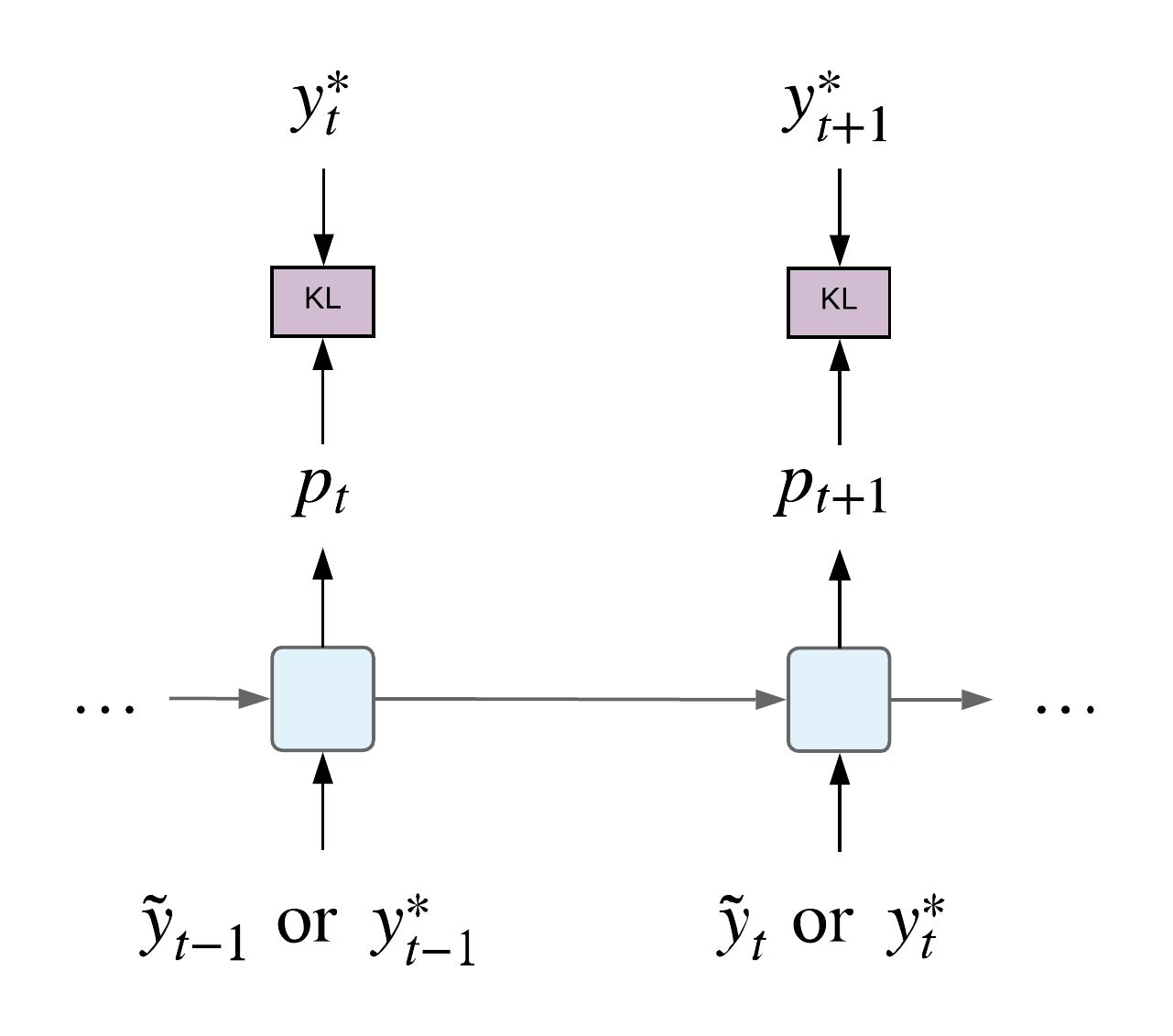}
                      \caption{Scheduled Sampling}
                      \label{subfig:ss}
                  \end{subfigure}
                  \begin{subfigure}[t]{0.5\textwidth}
                      \centering
                      \includegraphics[width=0.75\textwidth]{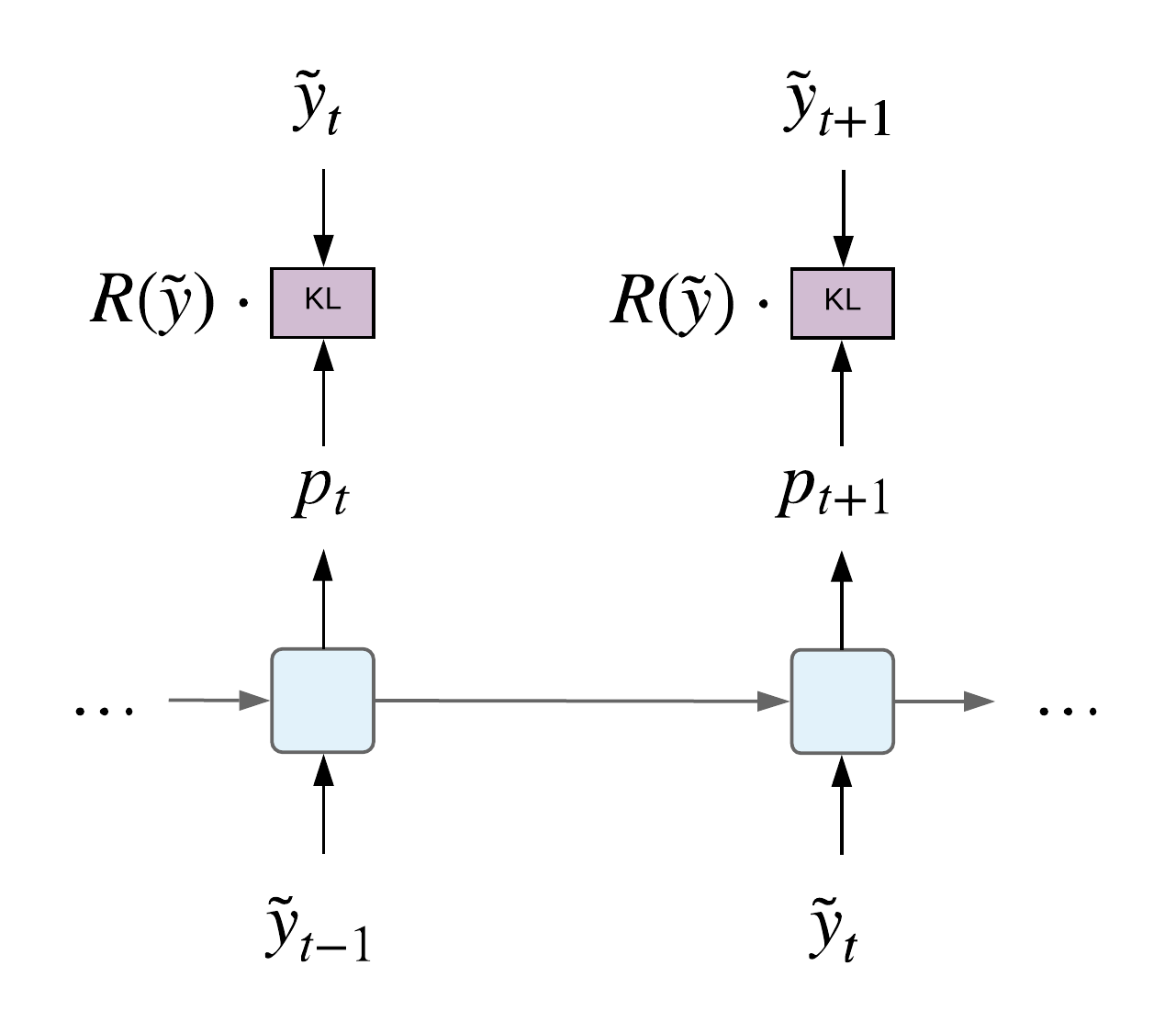}
                      \caption{Policy Gradient}
                      \label{subfig:policy}
                  \end{subfigure}%
                  \begin{subfigure}[t]{0.5\textwidth}
                      \centering
                      \includegraphics[width=0.75\textwidth]{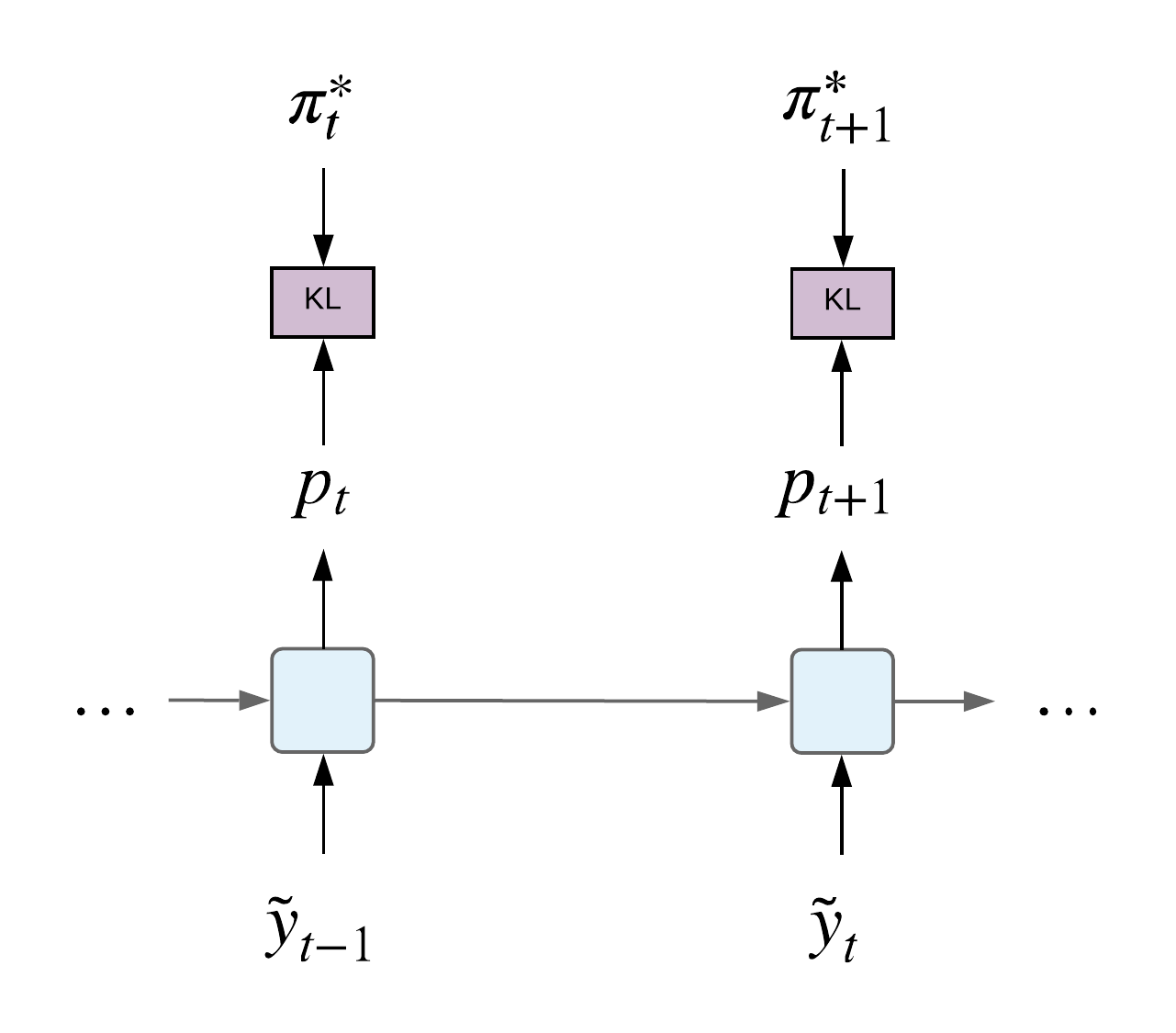}
                      \caption{Optimal Completion Distillation}
                      \label{subfig:ocd}
                  \end{subfigure}
                  \subref{subfig:teacher-forcing}
}
\caption{Illustration of different training strategies for autoregressive sequence models.
(a) Teacher Forcing: the model conditions on correct prefixes and is taught to predict the next ground truth token.
(b) Scheduled Sampling: the model conditions on tokens either from ground truth or drawn from the model and is taught to predict the next ground truth token regardless.
(c) Policy Gradient: the model conditions on prefixes drawn from the model and is encouraged to reinforce sequences with a large sequence reward $R(\tilde y)$.
(d) Optimal Completion Distillation: the model conditions on prefixes drawn from the model and is taught to predict an optimal completion policy $\pi^*$ specific to the prefix.}
\label{fig:trainingstrats}
\end{figure}

Unlike teacher forcing and Scheduled Sampling (SS), policy gradient approaches (\eg~\cite{ranzato-iclr-2016,bahdanau-iclr-2017})
and OCD aim to optimize the empirical reward objective \eqref{eq:er} on the training set. We illustrate four different training strategies
of MLE, SS, Policy Gradient and OCD in \figref{fig:trainingstrats}. The drawback of policy gradient techniques is twofold: 1) they cannot
easily incorporate ground truth sequence information except through the reward function, and 2) they have difficulty reducing the variance of the gradients to
perform proper credit assignment. Accordingly, most policy gradient approaches~\cite{ranzato-iclr-2016,bahdanau-iclr-2017,wu2016google} pre-train the model
using teacher forcing. By contrast, the OCD method proposed in this paper defines an optimal completion policy $\pi^*_t$
for any {\em off-policy} prefix by incorporating the ground truth information. Then, OCD optimizes a token level log-loss
and alleviates the credit assignment problem. Finally, training is much more stable, and we do not require initialization nor joint optimization with MLE.

There is an intuitive notion of {\em exposure bias}~\cite{ranzato-iclr-2016} discussed in the literature as a limitation of teacher forcing.
We formalize this notion as follows. One can think of the optimization of the log loss~\eqref{eq:cll} in an autoregressive models as a classification problem, where the input to the classifier is a tuple $(\bs, \bys_{<t})$ and the correct output is $\a^*_i$, where $\bas_{<t} \equiv (y^*_1, \ldots, y^*_{t-1})$. Then the training dataset comprises different examples and different prefixes of the ground truth sequence. The key challenge is that once the model is trained, one should not expect the model to generalize to a new prefix $\by_{<t}$ that does not come from the training distribution of $P(\bys_{<t})$. This problem can
become severe as $\by_{<t}$ becomes more dissimilar to correct prefixes. During inference, when one conducts beam search with a large beam size then
one is more likely to discover wrong generalization of $\pit(\hat{y}_{t} | \bai_{<t}, \bx)$, because the sequence is optimized globally.
A natural strategy to remedy this issue is to train on arbitrary prefixes. Unlike the aforementioned techniques OCD can train on any prefix
given its {\em off-policy} nature.

\comment{
The input trajectory that the model trains on introduces exposure bias to the model as it constrains the distribution of visited state and action pairs, e.g. generated prefixes. In teacher forcing the effect of exposure bias reveals by the significant performance decreases one observes when increasing the beam size in the beam search algorithm for decoding. Since OCD optimization is off-policy and is conditioned on the input prefix as the temporal dependency, we can train the model on a flexible range of trajectories. Rather than the typical $\epsilon$-greedy exploration in Q-learning we adopt a parameter-free online sampling regime where we always sample from the model distribution, $p_\theta(y_t | \ba_{<t})$, at each step of decoding during training. 
}
}

%% file: method.tex
\comment{
\begin{figure}[t]
  \centering
    \includegraphics[width=.5\textwidth]{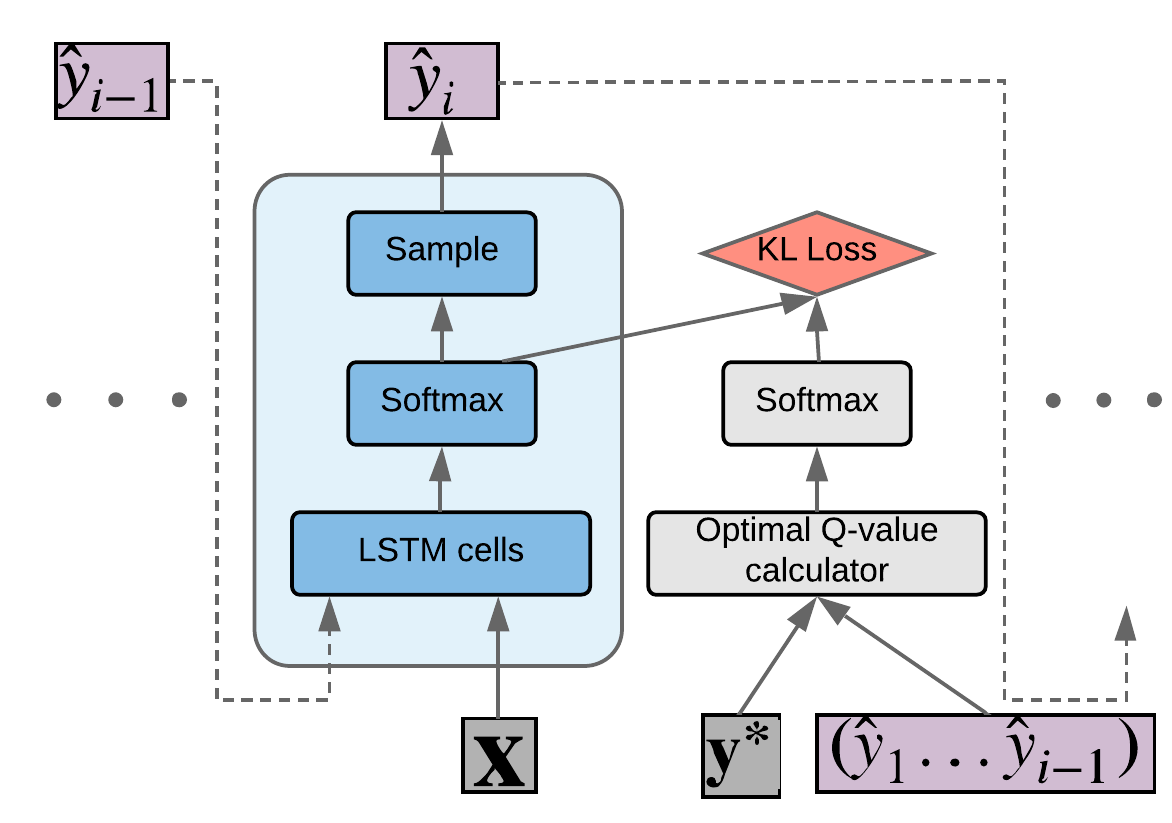}
  \caption{Diagram of different RNN cell input and loss strategies in the baseline seq2seq model vs the baseline with schedule sampling and the DOC approach.}
\label{fig:ocd}
\vspace{-0.4cm}
\end{figure}
}

\comment{Inspired by search based structured prediction~\citep{daumeetal09,ross-aistats-2011}
and policy distillation~\citep{rusu-iclr-2016},
we develop Optimal Completion Distillation (OCD) for optimizing seq2seq models based on edit distance.}
To alleviate the mismatch between inference and training, we {\em never} train on ground truth target sequences.
Instead, we always train on sequences generated by sampling from the current model that is being optimized.
Let $\byt$ denote a sequence generated by \rebut{sampling from} the current model, and $\bys$ denote the ground truth target.
Applying MLE to autoregressive models casts the problem of sequence learning as optimizing a mapping $(\bx, \bys_{<t}) \to \ys_t$
from ground truth prefixes to correct next tokens. By contrast, the key question that arises when training on model samples is the choice
of targets for learning a similar mapping $(\bx, \byt_{<t}) \to \,??$ from generated prefixes to next tokens. Instead of
using a set of pre-specified targets, OCD solves a prefix-specific problem to find optimal extensions that lead to the
best completions according to the task evaluation metric. Then, OCD encourages
the model to extend each prefix with the set of optimal choices for the next token.

Our notion of optimal completion depends on the task evaluation metric denoted $R(\cdot, \cdot)$, which measures
the similarity between two complete sequences, \eg~the ground truth
target \vs~a generated sequence. Edit distance is a common task metric.
Our goal in sequence learning is to train a model, which achieves high scores of $R(\bys,\byt)$.
Drawing \rebut{connection} with the goal of reinforcement learning~\citep{suttonbarto98}, let us recall the notion of optimal Q-values.
Optimal Q-values for a state-action pair $(s,a)$, denoted $Q^*(s,a)$, represent the maximum future reward that an agent can accumulate
after taking an action $a$ at a state $s$ by following with optimal subsequent actions. Similarly, we define Q-values
for a prefix $\byt_{<t}$ and the extending token $a$, as the maximum score attainable by concatenating $[\byt_{<t}, a]$
with an optimal suffix $\by$ to create a full sequence $[\byt_{<t}, a, \by]$. Formally,
\begin{equation}
   \forall a \in \V,\qquad Q^*(\byt_{<t}, a) ~=~ \mmax{\by \in \Y} \,R(\bys, [\byt_{<t}, a, \by])~.
\label{eq:reward}
\end{equation}
Then, the optimal extension for a prefix $\byt_{<t}$ can be defined as tokens that attain the maximal Q-values,
\ie~$\mathrm{argmax}_{a} Q^*(\byt_{<t}, a)$. This formulation allows for a prefix $\byt_{<t}$ to be sampled
on-policy from the model $p_\theta$, or drawn off-policy \rebut{in any way}. \tabref{tab:sample} includes an example ground truth target from the Wall Street Journal dataset and the corresponding
generated sample from a model. We illustrate that for some prefixes there exist more than a single optimal extension leading
to the same edit distance.

Given $Q$-values for our prefix-token pairs, we use an exponential transform followed by normalization 
to convert Q-values to a soft optimal policy over the next token extension, 
\begin{align}
    \pi^*(a \mid \byt_{<t}) = \frac{\exp ({Q^*(\byt_{<t}, a)}/\tau)}{\sum_{a'} \exp {(Q^*(\byt_{<t}, a')/\tau)}}~,
\label{eq:pis}
\end{align}
where $\tau \ge 0$ is a temperature parameter. Note the similarity of $\tau$ and the label smoothing parameter helpful within MLE.
In our experiments, we used the limit of $\tau \to 0$ resulting in hard targets and no hyper-parameter tuning.

{Given a training example $(\bx, \bys)$, we first draw a full sequence $\byt \sim p_\theta(\cdot\mid\bx)$ {\em i.i.d.} from the current model, 
and then minimize a per-step KL divergence between the optimal policy and the model distribution over the next token extension at each time step $t$. The OCD objective is
expressed as,
\comment{
\begin{equation}
 \mathcal{O}_{\text{OCD}}(\theta, (\bx, \bys), \byt) ~=~ \sum\nolimits_t \kl{\pi^*(\cdot \mid \byt_{<t})}{p_\theta(\cdot \mid \byt_{<t}, \bx)}~.
\end{equation}
}
\begin{eqnarray}
\mathcal{O}_{\text{OCD}}(\theta) &=& \E\nolimits_{(\bx, \bas) \sim p_\mathcal{D}} \E_{\byt \sim p_\theta(\cdot\mid\bx)} \sum\nolimits_{t=1}^{\lvert \byt \rvert} \kl{\pi^*(\cdot \mid \byt_{<t})}{\pitt(\cdot \mid \byt_{<t}, \bx)}~.
\label{eq:ocdll}
\end{eqnarray}
\comment{
\begin{equation}
    \exp ({Q^*(\byt_{<t}, a)}/\tau) = \mmax{a'} ~Q^*(\byt_{<t}, a') - Q^*(\byt_{<t}, a) + 1
\end{equation}

\begin{equation}
\mathcal{O}_{\text{OCD}}(\theta) = \sum\nolimits_{(\bx, \bas) \in \mathcal{D}} C_{\byt \sim p_\theta}(\bx, \bys).
\label{eq:ocdll}
\end{equation}
\begin{equation}
C_{\byt \sim p_\theta}(\bx, \bys) = \sum\nolimits_t \kl{\pi^*(\ra_t \mid \byt_{<t})}{p_\theta(\ra_t \mid \byt_{<t}, \bx)}~.
\end{equation}
}%
For every prefix $\byt_{<t}$, we compute the optimal $Q$-values and use \eqref{eq:pis} to construct the optimal policy distribution $\pi^*$.
Then, we distill the knowledge of the optimal policy for each prefix $\byt_{<t}$ into the parametric model using a KL loss.}
For the important class of sequence learning
problems where {\it edit distance} is the evaluation metric,
we develop a dynamic programming algorithm to calculate \rebut{optimal} $Q$-values exactly and efficiently for all prefixes of a sequence $\byt$, discussed below. 
\comment{
\rebut{
\begin{equation}
\mathcal{O}_{\text{AGG}}(\theta) ~=~ \E\nolimits_{(\bx, \bas) \sim p_\mathcal{D}} \E_{\byt \sim p_\theta(\cdot\mid\bx)} \sum\nolimits_{t=1}^{\lvert \byt \rvert} Q^*(\byt_{<t}, \byt_t)~.
\label{eq:ocdagg}
\end{equation}
}
}
\begin{table}[t]
\vspace*{-.2cm}
\begin{center}
\begin{tabular}{@{}ll@{}}
\toprule
Target sequence $\bys$ & \texttt{~a~s~\_~h~e~\_~t~a~l~k~s~\_~h~i~s~\_~w~i~f~e}\\[.1cm]
Generated sequence $\byt$ & \texttt{~a~s~\_~e~e~\_~t~a~l~k~s~\_~w~h~o~s~e~\_~w~i~f~e}\\[.1cm]
Optimal extensions for & \texttt{\color{blue} a~s~\_~h~e~\_~t~a~l~k~s~\_~h~i~i~s~\_~\_~w~i~f~e}\\
edit distance (OCD  & \texttt{\color{blue} ~~~~~~~~h~~~~~~~~~~~~~~~~~h~~~i~~~w}\\
targets)           & \texttt{\color{blue} ~~~~~~~~\_}\\
\comment{\hline
\hline
Hamming mismatch & a~s~\_~{\color{red}e}~e~\_~t~a~l~k~s~\_~{\color{red}wh~o~s~e~\_~w~i~f~e}\\
\hline
Edit distance credit assignment & a~s~\_~{\color{red}e}~e~\_~t~a~l~k~s~\_~{\color{red}w}h~{\color{red}o}~s~{\color{red}e}~\_~w~i~f~e\\ }
\bottomrule
\end{tabular}
\end{center}
\caption{A sample sequence $\bys$ from the Wall Street Journal dataset, where the model's prediction $\byt$ is not perfect. The optimal next characters for each prefix of $\byt$ based on edit distance are shown in blue. For example, for the prefix ``as\_e'' there are $3$ optimal next characters of ``e'', ``h'', and ``\_''. All of these $3$ characters when combined with proper suffixes will result in a total edit distance of $1$.
}\label{tab:sample}
\vspace*{-.2cm}
\end{table}

\comment{helps here and we use an almost zero $\tau = 0.001$ for hard targets. We keep the formulation more general to be consistent with~\citep{hinton-nips-2014,rusu-iclr-2016} and indicate how label smoothing can be done in this case. 
sequence of tokens $\by'$ that, when concatenated produces
a full sequence $[\byt_{<t}, a, \byt']$ closest to $\bys$. }
\comment{Unlike reinforcement learning, in supervised sequence prediction one has access to ground truth sequences, and one can afford to do some exhaustive search
to find the near optimal completions. A natural strategy is to search over the concatenations of the prefix and the ground truth sequence to find
a lower bound on $Q^*$-values as in learning to search literature.}

\comment{
Optimal Q-values are uniquely defined for Markov Decision Processes (MDPs) under mild assumptions.

For each prefix of the generated sequence, we identify the set of all optimal suffixes thatresult in minimum total edit distance, using an efficient dynamic programming algorithm.3.At each position, we minimize a cross-entropy loss to encourage the model to completeeach prefix by selecting one of the optimal next tokens.
}

\comment{Inspired by prior work on search based structured prediction~\citep{daumeetal09,ross-aistats-2011},
Q-learning~\citep{watkins1992q}, and policy distillation~\citep{rusu-iclr-2016},}
\comment{
Based on reinforcement learning~\citep{suttonbarto98} techniques, i.e. Q-learning, we present a novel solution to both of the problems
discussed above.
Recall that the optimal Q-value for a state-action pair in a RL environment, denoted $Q^*(s, a)$, represents
the maximum future reward that an agent can accumulate after taking an action $a$ at a state $s$ by following up with the optimal actions.
Optimal Q-values are uniquely defined for Markov Decision Processes (MDPs) under mild assumptions~\citep{suttonbarto98}.
If one has access to optimal Q-values for all states and actions, then one can construct the optimal policy by selecting the action
with maximum Q-value at each state. Q-learning improves the estimate of Q-values iteratively by resorting to bootstrapping based on 
the Bellman optimality equations.

\comment{
Q-learning is an optimistic reinforcement learning technique which for any finite Markov decision process determines the optimal stationary policy ($\pi^*$) of action selection at each state to maximize a total reward $R$ over all next steps \citep{watkins1992q}. In Q-learning the quality ($Q$) of an action $a$ in a state $s$ is the expectation of maximum attainable reward. The challenge in Q-learning is to estimate the Q-values for an optimal policy, which refers to as $Q^*(s, a)$. Q-learning estimates $Q^*$-values by bootstrapping future values\comment{ according to the Bellman optimality equations}.
}

To build robust autoregressive seq2seq models, we would like the model to be able to predict the best token given
any prefix of \new{TODO}

\acronym's goal is to generate the sequence $\hat{\by}$ one token at a time given a generated prefix $\hat{\by}_{<t}$ such that the final $\hat{\by}$ is similar to $\bys$ evaluated by a reward function $R(\bys, \hat{\by})$. In this setup, a state is a sequence prefix, and an action is selecting a token from the vocabulary.
We show that in this case when $R$ is edit distance\comment{ (possibly weighted)}, one can efficiently
compute $Q^*$-values without any bootstrapping or search.
\comment{Inspired by prior work~\cite{l2s2005,ross-aistats-2011,ranzato-iclr-2016},
we make connections between this problem and reinforcement learning,
and in the notion optimal $Q$-values used in the Q-learning algorithm.  The goal of Q-learning is to estimate the optimal value of each state-action pair,
denoted $Q^*(s,a)$, which represents the maximal achievable reward if one takes an action $a$ at a state $s$ and continues with the optimal policy.}

\comment{
\mohammad{this seems more related to SS vs. OCD -- let's move it to the end of the section} As an example of a flaw in teacher forcing the log probability of the correct sequence, consider the target sequence: ``{\it The cat is here.}'' and a character based model has predicted ``{\it A ca}'' as the first $4$ characters. In teacher forcing, the model is trained to predict $y^*_4=$`a' regardless of the model predictions so far. On the other hand, considering the characters generated so far the next token that can lead to the minimum {\it TER} is `t' at time step $4$. 
}
\comment{
Our high level motivation is that ideas from Q-learning and policy distillation can help mitigate the issues faced by current sequence learning algorithms. }

\comment{More generally, one can combine
ideas from \secref{sec:oced} with bootstrapping from Q-learning to develop a general variant of OCD, which is applicable to any sequence reward function.

. Furthermore, the exact $Q^*$ can be calculated at each state regardless of the model, therefore  we do not need bootstrappping anymore. Also, calculating $Q^*$ at each step solves the credit assignment problem that cripples the convergence of most reinforcement learning techniques.

In OCD we map the sequence prediction task to Q-learning by defining each state as the input prefix so far ($\ba_{<t}$), and the action $a$ is which token to choose as the next character, conditioned on a particular $(\bx, \bys)$ pair.

We define $Q^*(\ba_{<t}, a)$ as the maximum attainable reward at this state:
\begin{align}
    Q^*(\ba_{<t}, a) = \mmax{\ba' \in \mathcal{Y}} \,R([\ba_{<t}, a, \ba'], \bys),
\label{eq:reward}
\end{align}
where
}

Given the \comment{optimal }$Q^*$-values for all prefix-token pairs, we define the
optimal completion policy for a given prefix $\ba_{<t}$ as,
\begin{align}
    \pi_\tau^*(a \mid \ba_{<t}) = \frac{\exp ({Q^*(\ba_{<t}, a)}/\tau)}{\sum_{a'} \exp {(Q^*(\ba_{<t}, a')/\tau)}}~,
\end{align}
for a label smoothing term $\tau$.  We are not
convinced that smoothness helps here and we use an almost zero $\tau = 0.001$ for hard targets. We keep the formulation more general to be consistent with~\citep{rusu-iclr-2016} and indicate how label smoothing can be done in this case. 

We distill~\citep{hinton-nips-2014}
the optimal completion policy \comment{from our teacher} into our \comment{student} model.
For training, we sample a sequence from the model denoted $\byt$, and for each position $t$, we optimize
the Kullback-Leibler (KL) divergence between the soft optimal
completion policy $\pi^*$ and the prediction of the current model, expressed as,
\begin{equation}
 \mathcal{O}_{\text{OCD}}(\theta, \byt) ~=~ \sum\nolimits_t \kl{\pi_\tau^*(a \mid \byt_{<t})}{p_\theta(a \mid \byt_{<t})}~.
\end{equation}

}

\vspace*{-.2cm}
\subsection{Optimal Q-values for Edit Distance}
\vspace*{-.1cm}
\label{sec:oced}

We propose a dynamic programming algorithm to calculate optimal
Q-values exactly and efficiently for the reward metric of negative edit distance,
\ie~$R(\bys,\byt) = -\ed(\bys,\byt)$. Given two sequences $\bys$ and $\byt$, we compute the Q-values for
every prefix $\byt_{<t}$ and any extending token $a \in \V$ with an asymptotic complexity of $O(|\bys|.|\byt| + |\V|.|\byt|)$.
Assuming that $|\bys| \approx |\byt| \leq |\mathcal{V}|$, our algorithm does not increase the time
complexity over MLE, since 
computing the cross-entropy losses in MLE also requires
a complexity of $O(|\bys|.|\V|)$. \rebut{When this assumption does not hold, \eg~genetic applications, OCD is less efficient than MLE.
However, in practice, the wall clock time is dominated by the forward and backward passes of a neural networks,
and the OCD cost is often negligible. We discuss the efficiency of OCD further in Appendix~\ref{sec:algorithm}.}

\renewcommand{\arraystretch}{1.1}
\begin{table}[b]
\caption{Each row corresponds to a prefix of ``SATRAPY'' and shows
edit distances with all prefixes of ``SUNDAY''. We also show OCD targets (optimal extensions) for each prefix,
and minimum value along each row, denoted $m_i$ \rebut{(see~\eqref{eq:min-ed-bound})}.
We highlight the trace path for $\ed(\text{``Satrapy''}, \text{``Sunday''})$.}
\label{tab:edit}
\centering
\small
\begin{tabular}{@{}|>{\columncolor[HTML]{C0C0C0}}l|l|l|l|l|l|l|l|l|c|@{}}
\hline
\multicolumn{8}{|c|}{Edit Distance Table} & \acronym Targets & $m_i$\\ \hline
\rowcolor[HTML]{C0C0C0}
 &  & S & U & N & D & A & Y & & \\ \hline
 & \cellcolor{Ccolor}{\bf \textcolor{Fcolor}0} & 1 & 2 & 3 & 4 & 5 & 6 & S &  {\bf \textcolor{Fcolor}0}\\ \hline
S & 1 & \cellcolor{Ccolor}{\bf \textcolor{Fcolor}0} & 1 & 2 & 3 & 4 & 5 & U &  {\bf \textcolor{Fcolor}0}\\ \hline
A & 2 & {\bf\textcolor{Fcolor} 1} & \cellcolor{Ccolor}{\bf \textcolor{Fcolor} 1} & 2 & 3 & 3 & 4 & U, N &  {\bf \textcolor{Fcolor}1} \\ \hline
T & 3 & {\bf \textcolor{Fcolor}2} & {\bf\textcolor{Fcolor} 2} & \cellcolor{Ccolor}{\bf \textcolor{Fcolor}2} & 3 & 4 & 4 & U, N, D &  {\bf\textcolor{Fcolor}2} \\ \hline
R & 4 & {\bf \textcolor{Fcolor}3} & {\bf \textcolor{Fcolor}3} & {\bf \textcolor{Fcolor}3} & \cellcolor{Ccolor}{\bf\textcolor{Fcolor} 3} & 4 & 5 & U, N, D, A &  {\bf\textcolor{Fcolor}3}\\ \hline
A & 5 & 4 & 4 & 4 & 4 & \cellcolor{Ccolor}{\bf \textcolor{Fcolor}3} & 4 & Y &  {\bf\textcolor{Fcolor}3}\\ \hline
P & 6 & 5 & 5 & 5 & 5 & \cellcolor{Ccolor}{\bf \textcolor{Fcolor}4} & {\bf \textcolor{Fcolor}4} & Y, \textless{}/s\textgreater~ &  {\bf\textcolor{Fcolor}4}\\ \hline
Y & 7 & 6 & 6 & 6 & 6 & 5 & \cellcolor{Ccolor}{\bf\textcolor{Fcolor} 4} & \textless{}/s\textgreater~ &  {\bf\textcolor{Fcolor}4}\\ \hline
\end{tabular}
\vspace*{-.2cm}
\end{table}
\renewcommand{\arraystretch}{1}

Recall the Levenshtein algorithm \citep{levenshtein1966binary} for calculating the minimum number of edits (insertion, deletion and substitution) required to convert sequences $\byt$ and $\bys$ to each other based on,
\begin{equation}
\begin{cases}
    \ed(\byt_{<-1}, :) = \infty\\
    \ed(:, \bys_{<-1}) = \infty\\
    \ed(\byt_{<0},\bys_{<0}) = 0~,~
\end{cases}
    D_{\text{edit}}(\byt_{<i},\bys_{<j}) = \text{min}\begin{cases}
        D_{\text{edit}}(\byt_{<i-1},\bys_{<j}) + 1\\
        D_{\text{edit}}(\byt_{<i},\bys_{<j - 1}) + 1\\
        D_{\text{edit}}(\byt_{<i-1},\bys_{<j-1}) + \one{\yt_i \neq \ys_j}~.
    \end{cases}
\label{eq:levenshtein}
\end{equation}
\tabref{tab:edit} shows an example edit distance table for sequences ``Satrapy'' and ``Sunday''.
Our goal is to identify the set of all optimal suffixes $\by \in \Y$ that result in a full sequences $[\byt_{<i}, \by]$ with
a minimum edit distance \vs~$\bys$.
\comment{
\begin{equation}
    \min_{\by \in \Y} \ed([\byt_{<i}, \by], \bys) \ge \min_{0 \le j \le |\bys|} \,\ed(\byt_{<i},\bys_{<j})~,
\label{eq:min-ed-bound}
\end{equation}
because for any $\by$ is, according to \eqref{eq:levenshtein}, one can trace $\ed([\byt_{<i}, \by], \bys)$ 
to one of the entries within the row of the edit distance table corresponding to $\forall j\, \ed(\byt_{<i},\bys_{<j})$.
Because the updates in~\eqref{eq:levenshtein} are non-decreasing, hence \eqref{eq:min-ed-bound}. \mohammad{todo}}
\begin{lemma}
\label{th:suffix}
The edit distance resulting from any potential suffix $\by \in \Y$ is lower bounded by $m_i$, 
\begin{equation}
   \forall \by \in \Y,~~ \ed([\byt_{<i}, \by], \bys) ~\ge~ \min_{0 \le j \le |\bys|} \,\ed(\byt_{<i}, \bys_{<j})
   ~=~ m_i~.
\label{eq:min-ed-bound}
\end{equation}
\vspace*{-.4cm}
\end{lemma}
\begin{proof}
Let's consider the path $P$ that traces $\ed([\byt_{<i},\by], \bys)$ back to $\ed(\byt_{<0}, \bys_{<0})$ connecting each cell to an adjacent parent cell, which provides the minimum value among the three options in \eqref{eq:levenshtein}. Such a path for tracing edit distance between ``Satrapy'' and ``Sunday'' is shown in Table~\ref{tab:edit}.  Suppose the path $P$ crosses row $i$ at a cell $(i, k)$. Since the operations in \eqref{eq:levenshtein} are non-decreasing, the edit distance along the path cannot decrease, so $\ed([\byt_{<i},\by], \bys) \geq \ed(\byt_{<i}, \bys_{<k}) \ge m_i$.
\vspace*{-.2cm}
\end{proof}

\comment{
Since $D_{\text{edit}}(\byt_{<i},\bys_{<j})$ never decreases as $i$ or $j$ increase according to~\eqref{eq:levenshtein}. Hence,
\begin{equation}
  \forall a, \qquad  Q^*(\byt_{<i}, a) \leq -\mmin{j < |\bys|} \,D_{\text{edit}}(\byt_{<i},\bys_{<j}),
\end{equation}
Therefore, the suffix of $\bys_{<j}$ who has minimum edit distance with $\byt_{<i}$ are a subset of optimal completions for $\byt_{<i}$. Furthermore, we prove that
\begin{lemma}
\label{th:suffix}
The set of optimal completions is limited to only suffixes of $\bys$.
\begin{equation}
\forall s \in \mathcal{Y},~~~    s = \argmin{\byt' \in \mathcal{Y}} \, D_{\text{edit}}([\byt_{<i},\byt'], \bys) \implies \exists j,~~~ s = \bys_{j\leq} 
\end{equation}
\end{lemma}
\begin{proof}

By contradiction, assume
\begin{equation}
    \exists s \in \mathcal{Y},~~~    s = \argmin{\byt' \in \mathcal{Y}} \, D_{\text{edit}}([\byt_{<i},\byt'], \bys) , ~~~\forall j: D_{\text{edit}}([\byt_{<i},s], \bys) \leq D_{\text{edit}}([\byt_{<i},\bys_{j\leq}], \bys) 
\end{equation}
Among which, consider the sequence $s$ which has maximum shared suffix with $\bys$. Reverse both sequences. $D_{\text{edit}}([\overline{s},\overline{\byt}_{<i}],\overline{\bys}) = D_{\text{edit}}([\byt_{<i},s], \bys)$. Assume the first discrepancy between $\overline{s}$ and $\overline{\bys}$ happens at index $h$. Such $h$ exists otherwise $s$ would have been a suffix of $\bys$. Replacing $\overline{s}_h$ with $\overline{\bys}_h$ does not increase the edit distance. Despite that, it increases the shared suffix of $s$ with $\bys$ which contradicts with $s$ having the maximum shared suffix.
\end{proof}
}

Then, consider any $k$ such that $\ed(\byt_{<i}, \bys_{<k}) = m_i$. Let $\bys_{\geq k} \equiv (y^*_k, \ldots, y^*_{\lvert \bys \rvert})$ denote a suffix of $\bys$. We conclude that $\ed([\byt_{<i}, \bys_{\geq k}], \bys) = m_i$, because on the one hand there is a particular edit path that results in $m_i$ edits,
and on the other hand $m_i$ is a lower bound according to Lemma~\ref{th:suffix}. Hence any such $\bys_{\geq k}$ is an optimal suffix for $\byt_{<i}$. Further,
it is straightforward to prove by contradiction that the set of optimal suffixes is limited to suffixes $\bys_{\geq k}$ corresponding to
$\ed(\byt_{<i}, \bys_{<k}) = m_i$.

Since the set of optimal completions for $\byt_{<i}$ is limited to $\bys_{\geq k}$, the only extensions that can lead to maximum reward are the starting token of such suffixes ($\ys_k$). Since $\ed(\byt_{<i}, \bys_{<k}) = m_i$ as well, we can identify the optimal extensions by calculating the edit distances between all prefixes of $\byt$ and all prefixes of $\bys$ which can be efficiently calculated by dynamic programming in $\mathcal{O}(|\byt|.|\bys|)$. For a prefix $\byt_{<i}$ after we calculate the minimum edit distance $m_i$ among all prefixes of $\bys$, we set the $Q^*(\byt_{<i}, \ys_{k}) = -m_i $ for all $k$ where $\bys_{<k}$ has edit distance equal to $m_i$. We set the $Q^*$ for any other token to $-m_i-1$.
 We provide the details of our modified Levenshtein algorithm to efficiently compute the $Q^*(\byt_{<i}, a)$ for all $i$ and $a$ in Appendix~\ref{sec:algorithm}.

 \comment{
\begin{table}[b]
\caption{Each row corresponds to a prefix of ``SATRAPY'' and shows
edit distances with all prefixes of ``SUNDAY''. We also show OCD targets (optimal next tokens) for each prefix,
and minimum value along each row, denoted $m_i$.
We highlight the race path for $\ed(\text{``Satrapy''}, \text{``Sunday''})$.}
\label{tab:edit}
\vspace*{-.2cm}
\centering
\small
\begin{tabular}{@{}|>{\columncolor[HTML]{C0C0C0}}l|l|l|l|l|l|l|l|l|l|@{}}
\hline
\multicolumn{8}{|c|}{Edit Distance Table} & OCD targets & $m_i$\\ \hline
\rowcolor[HTML]{C0C0C0}
&  & S & U & N & D & A & Y &  \\ \hline
& \cellcolor{Ccolor}0 & 1 & 2 & 3 & 4 & 5 & 6 & S 
\\ \hline
S & 1 & \cellcolor{Ccolor}0 & 1 & 2 & 3 & 4 & 5 & U 
\\ \hline
A & 2 & \cellcolor{Ccolor}1 & 1 & 2 & 3 & 3 & 4 & U, N 
\\ \hline
T & 3 & \cellcolor{Ccolor}2 & 2 & 2 & 3 & 4 & 4 & U, N, D 
\\ \hline
U & 4 & 3 & \cellcolor{Ccolor}2 & 3 & 3 & 4 & 5 & N 
\\ \hline
R & 5 & 4 & 3 & \cellcolor{Ccolor}3 & 4 & 4 & 5 & N, D 
\\ \hline
D & 6 & 5 & 4 & 4 & \cellcolor{Ccolor}3 & 4 & 5 & A 
\\ \hline
A & 7 & 6 & 5 & 5 & 4 & \cellcolor{Ccolor}3 & 4 & Y 
\\ \hline
Y & 8 & 7 & 6 & 6 & 5 & 4 & \cellcolor{Ccolor}3 & \textless{}/s\textgreater
\\ \hline
\end{tabular}
\vspace*{-.2cm}
\end{table}
}

%% file: relatedwork.tex
\label{sec:relatedwork}
\comment{
OCD considers the evaluation reward as the optimization goal, similar to Reinforcement Learning  (RL) techniques. Furthermore, it can trace the reward of edit distance back to time steps in the model trajectory, resembling the Imitation Learning techniques. Also, OCD can calculate the step reward for all actions, close to the aim of the Learning to Search techniques. In contrast to prior work, by focusing on Edit Distance, OCD proposes an algorithm to calculate step reward for all actions more precisely and more efficiently. Also, OCD's formulation of the task (combination of optimal reward and KL) has lower variance which is apparent by its ability to train only on model samples unlike other methods. We now discuss OCD's relation to specific prior work in details.

\textbf{Trajectory Reward Optimization.} Approaches based on reinforcement-learning have been applied to sequence prediction problems,
including REINFORCE \citep{ranzato-iclr-2016}, Actor-Critic \citep{bahdanau-iclr-2017} and Self-critical Sequence Training \citep{rennie-cvpr-2017}.
These methods sample sequences from the model's distribution and backpropagate a sequence-level task objective (\eg~edit distance).
Beam Search Optimization \citep{wiseman-emnlp-2016} and Edit-based Minimum Bayes Risk (EMBR)~\citep{prabhavalkar-icassp-2018} is similar,
except the sampling procedure is replaced with beam search.
These training methods suffer from high variances and credit assignment problems.
\comment{The use of sequence-level cost makes it difficult for the model to learn the subtleties of the loss, and this is especially true for long sequences.}
By contrast, OCD takes advantage of the decomposition of the sequence-level objective into token level optimal completion targets.
This reduces the variance of the gradient and stabilizes the model.
Crucially, unlike most RL-based approaches, we neither need MLE pretraining or joint optimization with log-likelihood.

Reward Augmented Maximum Likelihood (RAML) \citep{norouzi-nips-2016} and related algorithms~\citep{koyamada2017,ma2017softmax,elbayad2018token,wang2018switchout} are also similar
to RL-based approaches. Instead of sampling from the model's distribution, RAML samples sequences from the true exponentiated reward distribution.
However, sampling from the true distribution is often difficult and intractable. \rebut{SPG \citep{ding2017cold} changes the policy gradient formulation to sample from a reward shaped model distribution. Therefore, its samples are closer than RAML samples to model samples. In order to facilitate sampling from their proposed distribution SPG points out the decomposability of their metrics (ROUGE). Although they have lower variance due to their biased samples, SPG suffers from the same problems as RAML and RL-based methods in credit assignment.}
\comment{for long sequences. Additionally, it is unclear whether sampling from the true distribution is the correct thing to do, since during inference, the model samples from the model's posterior.}

\textbf{Trajectory Reward Optimization per Step.} In order to solve the credit assignment problem, one can try to trace the final reward back into per step rewards. OCD lies in this category of methods and is inspired by Imitation Learning and Learning to Search (L2S)
techniques where a student policy is optimized with a policy at each step of the trajectory.
There are two questions to be addressed:

\textbf{Q1)} Given an expert teacher ($\pi^*$ and/or $Q^*$) how to optimize the student?

\textbf{A1.1)} DAgger \citep{ross-aistats-2011} answers this question and in particular is closely
related to OCD. DAgger iterates between augmenting ground truth with new model samples and mimicking expert's policy ($\pi^*$) on the full aggregated dataset. Similar to DAgger our goal is to mimic $\pi^*$ but only on model samples. Also, since in practice $\pi^*$ is only given for expert's trajectory we derive $\pi^*$ for new states based on reward ($Q^*$) and do not rely on having access to experts policy for new states.
\comment{\citep{sun2018truncated}}

Scheduled Sampling (SS)~\citep{bengio-nips-2015} and Data as Demonstrator (DaD)~\citep{dad2015} are instantiations of DAgger that aim to alleviate the mismatch between training and inference prefix distributions. SS generates prefixes by substituting some of the ground truth tokens with samples from the model via a tunable mixing schedule. SS uses the ground truth sequence as the training targets uninformed of the synthetic prefixes. Hence their targets are Hamming Distance optimal. By contrast, OCD draws tokens via $100\%$ sampling from the model, without any schedule. While SS can only handle substitution errors, OCD considers insertions and deletions as well by solving an alignment problem to find Edit Distance optimal targets.

\textbf{A1.2)} Another answer closely related to our setup is AggreVated \citep{sun2017deeply} which provides the online and natural derivatives of AggreVaTe \citep{ross2014reinforcement}, suitable for training with LSTMs. In this line of work, they assume access to $Q^*$ for model policy sampled action-state pairs and they still train on a mixture of ground truth and model trajerctories. Their objective is to minimize the per step cost-to-go $\mathop{\mathbb{E}}_{a \sim \pi(\cdot | s_t)}Q_t^*(s_t,a)$ and they define the online gradient as $\sum_{a}\nabla_{\theta}\pi(a|s_t;\theta) Q_t^*(s_t, a)$ which is close to optimizing reverse KL without optimizing the entropy. OCD while being similar in having a  $Q^*$ sensitive objective, optimizes the proper KL with the normalized exponentiation of $Q^*$ which is SGD friendly. Therefore, OCD is able
to have {\em roll-in} prefixes drawn only from the model. \cite{cheng2018convergence} showed that
mixing in ground truth samples is an essential regularizer for value
aggregation convergence in imitation learning. Furthermore, we calculate $Q^*$ directly for our task without assuming it exists or rolling out the expert. 

\textbf{Q2)} How to calculate $\pi^*$, $Q^*$ or $Q^\pi$ for rewards other than Hamming Distance in practice?

\textbf{A2.0)} \rebut{One may answer this with a structured reward specific algorithm like \cite{goldberg2012dynamic} where they provide an efficient dynamic oracle for parse tree generation with un-ordered but restricted edges and optimize a margin loss on most probable edge. OCD similarly provides a dynamic oracle but for edit distance on sequence learning and distills all the optimal targets.}

\textbf{A2.1)} One general answer is the L2S \citep{l2s2005}
literature where works like LoLS \citep{chang2015learning} and \cite{goodman2016noise}
estimate the cost of actions by examining
multiple {\em roll-outs} of a generated
prefix and they address which action to expand. SeaRNN~\citep{searnn2017} is an adoptation of L2S for RNNs which approximates the cost of
each token by computing the task loss for vocabulary size roll-outs
at each time step. Then in one setup they optimize the KL of $Q^\pi$ with the RNN. OCD has similar architectural components. Also our cost-sensitive loss is similarly the KL objective but with $Q^*$. Furthermore, the exploration makes SeaRNN's per step time complexity $O(VT)$ and difficult to scale to real world
datasets, where either the sequences are long or the
vocabulary is large. While SeaRNN is a general approach, OCD exploits the special structure in edit distance
and find $Q^*$ by per step complexity of $O(V+T)$. Hence, unlike L2S and SeaRNN, which require
ground truth prefixes to stabilize training, we are able to train only on model samples.

\textbf{A2.2)} Another answer more closely related to our work is Policy
Distillation~\citep{rusu-iclr-2016}, where a deep Q-Network (DQN)
agent~\citep{mnihetal15} is optimized as the teacher first. Then, action sequences are sampled from
the teacher and the learned Q-value estimates are distilled~\citep{hinton-nips-2014} into a
smaller student network using a KL loss. OCD adopts a similar loss function,
but rather than estimating Q-values using bootstrapping,
we estimate exact $Q$-values using dynamic programming. Moreover, we draw
samples from the student rather than the teacher.

\cite{bahdanau-iclr-2016} also noticed some of the nice structure of edit distance, but they optimize the model
by regressing its outputs to edit distance values leading to suboptimal performance. Rather, we first construct the optimal policy and then use knowledge distillation for training. \rebut{\cite{karita-icassp-2018} also noticed the decomposability of the edit distance cost function, however they only used this to reweight the policy gradients (likely-hood of the prefix) at a token level rather than a cost-sensitive classification of actions. Therefore, they still suffered from high variance and had to resort to MLE pretraining.}

Generally, \acronym excels at training from scratch, which makes it an ideal substitution for MLE. Hence, OCD is orthogonal to methods which require
MLE pretraining or joint optimization.
}

Our work \rebut{builds upon} Learning to Search \citep{l2s2005} and Imitation Learning
techniques~\citep{ross-aistats-2011,ross2014reinforcement,sun2018truncated},
where a student policy is optimized to imitate an expert teacher.
DAgger \citep{ross-aistats-2011} in particular is closely
related, where a dataset of trajectories \rebut{from an expert teacher is
aggregated with samples from past student models, 
and a policy is optimized to mimic a given expert policy $\pi^*$ at various states.
Similarly, OCD aims to mimic an optimal policy $\pi^*$ at all prefixes, but in OCD, the behavior policy is directly obtained from an online student. Further, the oracle policy is not provided during training, and we obtain the optimal policy by finding optimal $Q$-values.
AggreVaTeD \citep{sun2017deeply} assumes access to an unbiased estimate of Q-values and relies on variance reduction techniques and conjugate gradients
to address a policy optimization problem.} OCD calculates
exact Q-values and uses regular SGD for optimization. Importantly, our {\em roll-in} prefixes are drawn only from the student model,
and we do not require mixing in ground truth (\aka~expert) samples. \cite{cheng2018convergence} showed that
mixing in ground truth samples is an essential regularizer for value aggregation convergence in imitation learning.

\rebut{
Our work is closely related to Policy
Distillation~\citep{rusu-iclr-2016}, where a Deep Q-Network (DQN)
agent~\citep{mnihetal15} that is previously optimized is used as the expert teacher.}
Then, action sequences are sampled from
the teacher and the learned Q-value estimates are distilled~\citep{hinton-nips-2014} into a
smaller student network using a KL loss. OCD adopts a similar loss function,
but rather than estimating Q-values using bootstrapping,
we estimate exact $Q$-values using dynamic programming. Moreover, we draw
samples from the student rather than the teacher.
 
\rebut{
Similar to OCD, the learning to search (L2S)
techniques such as LOLS \citep{chang2015learning} and \cite{goodman2016noise}
also attempt to estimate the Q-values for each state-action pair. Such techniques examine}
multiple {\em roll-outs} of a generated
prefix \rebut{and aggregate the return values}. SeaRNN~\citep{searnn2017} approximates the \rebut{cost-to-go} for
each token by computing the task loss for \rebut{as many roll-outs as the vocabulary size}
at each time step \rebut{with a per step complexity of $O(VT)$. It is often difficult to scale approaches based on
multiple roll-outs} to real world
datasets, where either the sequences are long or the
vocabulary is large. 
OCD exploits the special structure in edit distance
and find exact Q-values efficiently \rebut{in $O(V+T)$ per step}. Unlike L2S and SeaRNN, which require
ground truth prefixes to stabilize training, we solely train on model samples.

Approaches based on Reinforcement Learning (RL) have also been applied to sequence prediction problems,
including REINFORCE \citep{ranzato-iclr-2016}, Actor-Critic \citep{bahdanau-iclr-2017} and Self-critical Sequence Training \citep{rennie-cvpr-2017}.
These methods sample sequences from the model's distribution and backpropagate a sequence-level task objective (\eg~edit distance).
Beam Search Optimization \citep{wiseman-emnlp-2016} and Edit-based Minimum Bayes Risk (EMBR)~\citep{prabhavalkar-icassp-2018} is similar,
but the sampling procedure is replaced with beam search.
These training methods suffer from high variances and credit assignment problems.
By contrast, OCD takes advantage of the decomposition of the sequence-level objective into token level optimal completion targets.
This reduces the variance of the gradient and stabilizes the model.
Crucially, unlike most RL-based approaches, we neither need MLE pretraining or joint optimization with log-likelihood.
\cite{bahdanau-iclr-2016} also noticed some of the nice structure of edit distance, but they optimize the model
by regressing its outputs to edit distance values leading to suboptimal performance. Rather, we first construct the optimal policy and then use knowledge distillation for training. \rebut{Independently, \cite{karita-icassp-2018} also decomposed edit distance into the contribution of individual tokens and used this decomposition within the EMBR framework. That said, \cite{karita-icassp-2018} do not theoretically justify this particular choice of decomposition
and report high variance in their gradient estimates.}

Reward Augmented Maximum Likelihood (RAML) \citep{norouzi-nips-2016} and its variants~\citep{ma2017softmax,elbayad2018token,wang2018switchout} are also similiar
to RL-based approaches. Instead of sampling from the model's distribution, RAML samples sequences from the true exponentiated reward distribution.
However, sampling from the true distribution is often difficult and intractable. RAML suffers from the same problems as RL-based methods in credit assignment. \rebut{SPG \citep{ding2017cold} changes the policy gradient formulation to sample from a reward shaped model distribution. Therefore, its samples are closer than RAML to the model's samples. In order to facilitate sampling from their proposed distribution SPG provides a heuristic to decompose ROUGE score. Although SPG has a lower variance due to their biased samples, it suffers from the same problems as RAML and RL-based methods in credit assignment.}

Generally, \acronym excels at training from scratch, which makes it an ideal substitution for MLE. Hence, OCD is orthogonal to methods which require
MLE pretraining or joint optimization.

%% file: experiments.tex
\section{Experiments}
\label{sec:experiments}

We conduct our experiments on speech recogntion on the Wall Street Journal (WSJ)~\citep{paul1992design} and Librispeech \citep{librispeech}
benchmarks. We only compare end-to-end speech recognition approaches that do not incorporate language model rescoring.
On both WSJ and Librispeech, our proposed {\bf OCD} (Optimal Completion Distillation) algorithm significantly outperforms our own strong baselines
including {\bf MLE} (Maximum Likelihood Estimation with label smoothing) and {\bf SS} (scheduled sampling with a well-tuned schedule).
Moreover, OCD significantly outperforms all prior work, achieving a new state-of-the-art on two competitive benchmarks.
\subsection{Wall Street Journal}
\begin{figure}[t]
\begin{minipage}{.48\textwidth}
\centering
  \includegraphics[height=.54\linewidth]{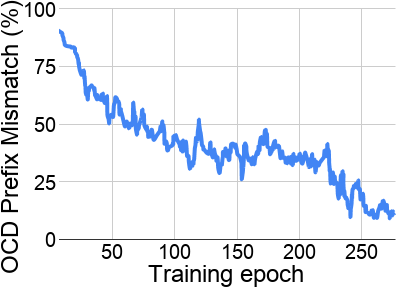}
  \vspace*{.3cm}
  \captionof{figure}{Fraction of \acronym training prefix tokens on WSJ which does not match ground truth.}
  \label{fig:frac}
\end{minipage}
\hspace{.4cm}
\begin{minipage}{.48\textwidth}
\centering
  \includegraphics[height=.54\linewidth]{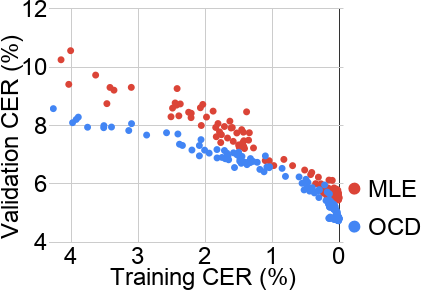}
  \vspace*{.3cm}
  \captionof{figure}{WSJ validation Character Error Rate (CER) per training CER for MLE and \acronym.}
  \label{fig:gen}
\end{minipage}%
\end{figure} 
\begin{table}
\centering
\caption{WSJ Character Error Rate (CER) and Word Error Rate (WER) of different baselines. Schedule Sampling optimizes for \textit{Hamming} distance and mixes samples from the model and ground truth with a \textit{probability} schedule (start-of-training $\rightarrow$ end-of-training). \acronym \textit{always} samples from the model and optimizes for \textit{all} characters which minimize \textit{Edit} distance. Optimal Completion Target optimizes for \textit{one} character which minimizes edit distance and another criteria (shortest or same \#words). }
\label{tab:comparison}
\begin{tabular}{lcc} 
 \toprule
 \textbf{Training Strategy} & \textbf{CER} & \textbf{WER} \\ 
 \midrule
 Schedule Sampling $(1.0 \rightarrow 1.0)$ & 12.1 & 35.6 \\
 Schedule Sampling $(0.0 \rightarrow 1.0)$ & 3.8 & 11.7 \\
 Schedule Sampling $(0.0 \rightarrow 0.55)$ & 3.6 & 10.2 \\
 \midrule
 Optimal Completion Target (Shortest) & 3.8 & 12.7\\
 Optimal Completion Target (Same \#Words) & 3.3 & 10.2\\
 \midrule
 Optimal Completion Distillation & \bf{3.1} & \bf{9.3} \\
 \bottomrule
\end{tabular}
\end{table}
\begin{table}[t]
\centering
\caption{Character Error Rate (CER) and Word Error Rate (WER) results on the end-to-end speech recognition WSJ task. We report results of our Optimal Completion Distillation (\acronym) model, and well-tuned implementations of maximum likelihood estimation (MLE) and Scheduled Sampling (SS).}
\label{tab:results}
\begin{tabular}{lcc} 
 \toprule
 \textbf{Model} & \textbf{CER} & \textbf{WER} \\ 
 \midrule
 Prior Work \\
 \qquad CTC \citeauthoryear{graves-icml-2014} & 9.2 & 30.1 \\
 \qquad CTC + REINFORCE \citeauthoryear{graves-icml-2014} & 8.4 & 27.3 \\
 \qquad Gram-CTC \citeauthoryear{liu-icml-2017} & - & 16.7 \\
 \qquad seq2seq \citeauthoryear{bahdanau-icassp-2016} & 6.4 & 18.6 \\
 \qquad seq2seq + TLE \citeauthoryear{bahdanau-iclr-2016} & 5.9 & 18.0 \\
 \qquad seq2seq + LS \citeauthoryear{chorowski-interspeech-2017} & - & 10.6 \\
 \qquad seq2seq + CNN \citeauthoryear{zhang-icassp-2017} & - & 10.5 \\
 \qquad seq2seq + LSD \citeauthoryear{chan-iclr-2017} & - & 9.6 \\
 \qquad seq2seq + CTC \citeauthoryear{kim-icassp-2017} & 7.4 & - \\
 \qquad seq2seq + TwinNet \citeauthoryear{serdyuk-iclr-2018} & 6.2 & - \\
 \qquad seq2seq + MLE + REINFORCE \citeauthoryear{tjandra-icassp-2018} & 6.1 & - \\
 \midrule
 Our Implementation \\
 \qquad seq2seq + MLE & 3.6 & 10.6 \\
 \qquad seq2seq + SS & 3.6 & 10.2 \\
 \qquad {\bf seq2seq + \acronym} & {\bf 3.1} & {\bf 9.3} \\
 \bottomrule
\end{tabular}
\end{table}
The WSJ dataset is readings of three separate years of the Wall Street Journal.
We use the standard configuration of si$284$ for training, dev$93$ for validation and report both
test Character Error Rate (CER) and Word Error Rate (WER) on eval$92$.
We tokenize the dataset to English characters and punctuation. 
Our model is an attention-based seq2seq network with a deep convolutional frontend as used in~\cite{zhang-icassp-2017}.
During inference, we use beam search with a beam size of $16$ for all of our models. We describe the architecture and hyperparameter details in Appendix \ref{sec:architecture}. 
We first analyze some key characteristics of the OCD model separately, and then compare our results with other baselines and state-of-the-art methods.

{\bf Training prefixes and generalization}. We emphasize that during training, the generated prefixes sampled from the model do not match the ground truth sequence, even at the end of training.
We define OCD prefix {\em mismatch} as the fraction of OCD training tokens that do not match corresponding ground truth training tokens at each position.
Assuming that the generated prefix sequence is perfectly matched with the ground truth sequence, then the OCD targets would simply be the following tokens of the ground truth sequence.
Hence, OCD becomes equivalent to MLE. 
\figref{fig:frac} shows that OCD prefixes mismatch is more than $25\%$ for the most of the training.
This suggests that OCD and MLE are training on very different input prefix trajectories.
Further, \figref{fig:gen} depicts validation CER as a function of training CER for different model checkpoints during training,
where we use beam search on both training and validation sets to obtain CER values.
Even at the same training CER, we observe better validation error for OCD, which suggests that OCD improves generalization of MLE, possibly
because OCD alleviates the mismatch between training and inference.

\comment{
\begin{figure}[t]

\centering
\begin{subfigure}{.5\textwidth}
  \centering
  \includegraphics[width=\linewidth]{train_cer_3}
  \caption{Training CER}
\end{subfigure}%
\begin{subfigure}{.5\textwidth}
  \centering
  \includegraphics[width=\linewidth]{val_cer_3}
  \caption{Validation CER}
\end{subfigure}
\caption{Log scale plots of training and validation Character Error Rate (CER). The naive sampling model always samples from the model, while the MLE model always feeds the ground truth as the input to the RNN. Both of these models are trained to match the ground truth at each sequence step. Although the \acronym module is always sampling from the model, it doesn't suffer like the Naive model and quickly catches the baselines and continues to further improve.
}
\label{fig:cer}
\end{figure} }
{\bf Impact of edit distance.} We further investigate the role of the optimizer by experimenting with different losses. Table~\ref{tab:comparison} compares the test CER and WER of the schedule sampling with a fixed probability schedule of $(1.0 \rightarrow 1.0)$ and \acronym model. Both of the models are trained only on sampled trajectories. The main difference is their optimizers, where the SS$(1.0\rightarrow1.0)$ model is optimizing the log likelihood of ground truth (\aka~Hamming distance). The significant drop in CER of SS$(1.0\rightarrow1.0)$ emphasizes the necessity of pretraining or joint training with MLE for models such as SS. \acronym is trained from random initialization and does not require MLE pretraining, nor does it require joint optimization with MLE. We also emphasize that unlike SS, we do not need to tune an exploration schedule, OCD prefixes are simply always sampled from the model from the start of training. We note that even fine tuning a pre-trained SS model which achieves $3.6\%$ CER with $100\%$ sampling increases the CER to $3.8$\%. This emphasizes the importance of making the loss a function of the model input prefixes, as opposed to the ground truth prefixes. Appendix~\ref{sec:hamm} covers another aspect of optimizing Edit distance rather than Hamming distance.

{\bf Target distribution.} Another baseline which is closer to MLE framework is selecting only one correct target. Table~\ref{tab:comparison} compares OCD with several Optimal Completion Target (OCT) models. In OCT, we optimize the log-likelihood of one target, which at each step we pick dynamically based on the minimum edit distance completion similar to OCD. We experiment with several different strategies when there is more than one character that can lead to minimum CER. In the OCT (Shortest), we select the token that would minimize the CER and the final length of the sequence. In the OCT (Same \#Words), we select the token that in addition to minimum CER, would lead to the closest number of words to the target sequence. We show that \acronym achieves significantly better CER and WER over the other optimization strategies compared in Table~\ref{tab:comparison}. This highlights the importance of optimizing for the entire set of optimal completion targets, as opposed to a single target.
\comment{
Table \ref{tab:results} summarizes prior work and our experimental results. First, we note that our baseline seq2seq model trained with MLE achieves a $3.6$\% CER and $10.6$\% WER, considerably better than most published prior work. We attribute this to very effective hyperparameter tuning. We also report a SS model of $3.6$\% CER and $10.2$\% WER.

\new{The SS baseline is trained with a schedule that starts off with MLE training, then transitions to conditioning samples generated by the model. The schedule linearly increases the sampling probability from $0$ to $100$\% over $1$M training steps. The best validation WER is achieved at $600$k steps where the sampling probability is just under $70$\%. We find that SS is unable to improve over our MLE baseline in terms of CER.}
The scheduled sampling baseline feeds the ground truth with zero probability of sampling from the model at the start of the training. Then linearly increases the probability of sampling from the model to reach $100\%$ sampling at $1000$K training iterations. However, the best validation WER is achieved at $600$K when the sampling probability hasn't reached $70\%$ yet. We find that SS does not improve over our MLE baseline in terms of CER.
}

{\bf State-of-the-art.} Our model trained with \acronym optimizes for CER; we achieve 3.1\% CER and 9.3\% WER, substantially outperforming our baseline by 14\% relatively on CER and 12\% relatively on WER. In terms of CER, our work substantially outperforms prior work as compared in Table~\ref{tab:results}, with the closest being \cite{tjandra-icassp-2018} trained with policy gradients on CER. In terms of WER, our work is also outperforming \cite{chan-iclr-2017}, which uses subword units while our model emits characters. \comment{We note the much larger gap and improvement in CER over WER, this is due to the fact that in our optimization, we optimized for CER rather than WER. However, it is still encouraging to see that improving CER correlates to improvement in WER.}
\comment{
\begin{figure}
\centering
\includegraphics[width=.8\textwidth]{beamskew.png}
\caption{The effect of beam search size on WER on WSJ.}
\label{fig:beamsearch}
\end{figure}}
\subsection{Librispeech} 
\begin{figure}[t]
\centering
\includegraphics[width=\textwidth]{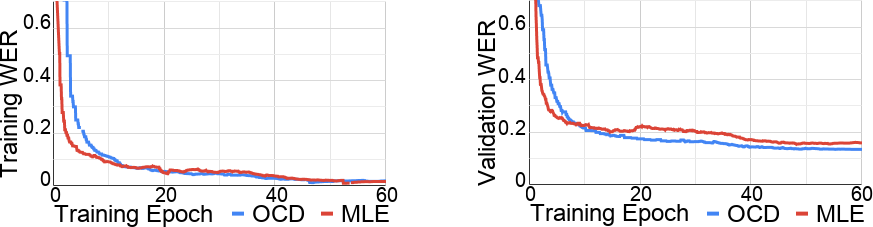}
\caption{\rebut{Librispeech training and validation WER per training epoch for OCD and MLE.}}
\label{fig:both}
\end{figure}
\begin{table}[b]
\centering 
\caption{Character Error Rate (CER) and Word Error Rate (WER) on LibriSpeech test sets.}
\label{tab:lib}
\begin{tabular}{@{}lcccc@{}}
\toprule
\multirow{2}{*}{\textbf {Model}} & \multicolumn{2}{c}{\textbf{test-clean}} & \multicolumn{2}{c}{\textbf{test-other}} \\
 & \textbf{CER} & \textbf{WER} & \textbf{CER} & \textbf{WER} \\ \midrule
 Prior Work \\
 \qquad Wav2letter \citep{collobert2016wav2letter} & 6.9 & 7.2 & - & - \\
 \qquad Gated ConvNet \citep{liptchinsky-arxiv-2017} & - &  6.7 & - & 20.8 \\
 \qquad Cold Fusion \citep{sriram2017cold} & 3.9 & 7.5 & 9.3 & 17.0 \\
 \qquad Invariant Representation Learning \citep{liang2018learning} & 3.3 & - & 11.0 & - \\
 \qquad Pretraining+seq2seq+CTC  \citep{zeyer2018improved} & - &  4.9 & - & 15.4\\
 \midrule
 Our Implementation \\
 \qquad seq2seq + MLE & 2.9 & 5.7 & 8.4 & 15.4\\
 \qquad {\bf seq2seq + OCD} & {\bf 1.7} & {\bf \libclean} & \bf{6.4} & \bf{\libother}\\
 \bottomrule
\end{tabular}
\end{table}
For the Librispeech dataset, we train on the full training set (960h audio data) and validate our results on the dev-other set. We report the results both on the ``clean'' and ``other'' test set. We use Byte Pair Encoding (BPE) \citep{sennrich-acl-2016} for the output token segmentation. BPE token set is an open vocabulary set since it includes the characters as well as common words and n-grams. We use 10k BPE tokens and report both CER and WER as the evaluation metric. We describe the architecture and hyperparameter details in Appendix \ref{sec:architecture}.

\rebut{Fig.~\ref{fig:both} shows the validation and training WER curves for MLE and OCD. OCD starts outperforming MLE on training decodings after training for 13 epochs and on validation decodings after 9 epochs.} Our MLE baseline achieves $5.7\%$ WER, while OCD achieves $\libclean\%$ WER on test-clean ($21\%$ improvement) and improves the state-of-the-art results over \cite{zeyer2018improved}.   test-other is the more challenging test split ranked by the WER of a model trained on WSJ \citep{librispeech} mainly because readers accents deviate more from US-English accents. On test-other our MLE baseline achieves 15.4\%, while our OCD model achieves \libother\% WER, outperforming the 15.4\% WER of \cite{zeyer2018improved}. Table~\ref{tab:lib} compares our results with other recent works and the MLE baseline on Librispeech.

%% file: conclusion.tex
\section{Conclusion}
\label{sec:conclusion}
This paper presents Optimal Completion Distillation (\acronym), a training procedure for optimizing autoregressive sequence models
base on edit distance. OCD is applicable to on-policy or off-policy trajectories, and in this paper,
we demonstrate its effectiveness on samples drawn from the model in an online fashion.
Given any prefix, \acronym creates an optimal extension policy by computing the exact optimal Q-values via dynamic programming.
The optimal extension policy is distilled by minimizing a KL divergence between the optimal policy and the model.
\acronym does not require MLE initialization or joint optimization with conditional log-likelihood.
\comment{The ideas presented here can be generalized to other evaluation metrics such as BLEU score, when bootstrapping is used to estimate
Q-values off-line. }\acronym achieves $3.1$\% CER and $9.3$\% WER on the competitive WSJ speech recognition task, and $\libclean\%$ WER on Librispeech
without any language model.
\acronym outperforms all published work on end-to-end speech recognition, including our own well-tuned MLE and scheduled sampling baselines
without introducing new hyper-parameters.


%% file: acknowledge.tex
\section*{Acknowledgements}
We thank Geoffrey Hinton for his valuable feedback and reviews. We thank Samy Bengio, Navdeep Jaitly,
and Jamie Kiros for their help in reviewing the manuscript as well. We thank Zhifeng Chen and
Yonghui Wu for their generous technical help.

%% file: supp.tex
\renewcommand\thefigure{\thesection.\arabic{figure}}
\renewcommand\thetable{\thesection.\arabic{table}}

\section{OCD Algorithm}
\label{sec:algorithm}
\begin{algorithm}[h!]
\caption[Caption of Alg]{EditDistanceQ op returns {\bf Q-values} of the tokens at each time step based on the minimum edit distance between a reference sequence $r$ and a hypothesis sequence $h$ of length $t$.
}\label{alg:edit}
\begin{algorithmic}[1]
\For{$j$ in $(0..t)$}
\State $d_j \gets j + 1$
\EndFor
\For{$i$ in $(1..t)$}
\State minDist $\gets i$
\State subCost $\gets i - 1$
\State insCost $\gets i + 1$
\For{$j$ in $(0..t - 1)$}
\If {$h_{i - 1}$ = $r_j$}
\State repCost $\gets 0$
\Else
\State repCost $\gets 1$
\EndIf
\State cheapest $\gets min($subCost $+$ repCost, $d_{j} + 1$, insCost)
\State subCost $\gets d_j$
\State insCost $\gets$ cheapest + 1
\State $d_j \gets $ cheapest
\If {$d_j < $ minDist}
\State minDist  $\gets d_j$
\EndIf
\EndFor
\If{minDist $= i$}
\State $Q_{i, r_1} \gets 1$
\EndIf
\For{$j$ in $(1..t)$}
\If{$d_j = $ minDist}
\State $Q_{i,r_{j+1}} \gets 1$
\EndIf
\EndFor
\For{all tokens k}
\State $Q_{i,k} \gets Q_{i,k} - 1 - $ minDist
\EndFor
\EndFor
\Return $Q$
\end{algorithmic}
\end{algorithm}
{\bf Complexity.} The total time complexity for calculating the sequence loss using \acronym is $O(T^2 + |V|T)$ where $V$ is the vocabulary size and $T$ is the sequence length. MLE loss has a time complexity of $O(|V|T)$ for calculating the softmax loss at each step. Therefore, assuming that $O(T) \leq O(|V|)$ \acronym does not change the time complexity compared to the baseline seq2seq+MLE. The memory cost of the OCD algorithm is $O(T + |V|T) = O(|V|T)$, $O(T)$ for the dynamic programming in line 4 - line 13 of Proc.~\ref{alg:edit} and $O(|V|T)$ for storing the stepwise $Q$ values. MLE also stores the one-hot encoding of targets at each step with a cost of $O(|V|T)$. Therefore, the memory complexity does not change  compared to the MLE baseline either.

Although the loss calculation has the same complexity as MLE, online sampling from the model to generate the input of next RNN cell (as in \acronym and SS) is generally slower than reading the ground truth (as in MLE). Therefore, overall a naive implementation of \acronym is $\leq 20\%$ slower than our baseline MLE in terms of number of step time.  However, since \acronym is stand alone and can be trained off-policy, we can also train on stale samples and untie the input generation worker from the training workers. In this case it is as fast as the MLE baseline.

\setcounter{figure}{0} 
\setcounter{table}{0} 

{\bf Run through.} As an example of how this algorithm works, consider the sequence ``SUNDAY'' as reference and ``SATURDAY'' as hypothesis. Table~\ref{tab:algorithm} first shows how to extract optimal targets and their respective $Q^*$-values from the table of edit distances between all prefixes of reference and all prefixes of hypothesis. At each row highlighted cells indicate the prefixes which has minimum edit distance in the row. The next character at these indices are the Optimal targets for that row. At each step the $Q^*$-value for the optimal targets is negative of the minimum edit distance and for the non-optimal characters it is one smaller.

Table~\ref{tab:algorithm} also illustrates how appending the optimal completions for the prefix ``SA'' of the hypothesis can lead to the minimum total edit distance. Concatenating with both reference suffixes, ``UNDAY'' and ``NDAY'' will result in an edit distance of $1$. Therefore, predicting ``U'' or ``N'' at step 2 can lead to the maximum attainable reward of $(-1)$. 

\begin{table}[h]
\caption{Top: Each row corresponds to a prefix of ``SATURDAY'' and shows
edit distances with all prefixes of ``SUNDAY'', along with the optimal targets and their $Q^*$-value at that step. The highlighted cells indicate cells with minimum edit distance at each row. Bottom: An example of appending suffixes of ``SUNDAY'' with minimum edit distance to the prefix ``SA''.}
\label{tab:algorithm}
\centering
\begin{tabular}{@{}|>{\columncolor[HTML]{C0C0C0}}l|l|l|l|l|l|l|l|l|c|@{}}
\hline
 & \multicolumn{7}{c|}{Edit Distance} & OCD Targets & Q-values\\ \hline
\rowcolor[HTML]{C0C0C0}
 &  & S & U & N & D & A & Y & &  \\ \hline
 & \cellcolor[HTML]{FE996B}0 & 1 & 2 & 3 & 4 & 5 & 6 & S & 0 \\ \hline
S & 1 & \cellcolor[HTML]{FE996B}0 & 1 & 2 & 3 & 4 & 5 & U & 0 \\ \hline
A & 2 & \cellcolor[HTML]{FE996B}1 & \cellcolor[HTML]{FE996B}1 & 2 & 3 & 3 & 4 & U, N & -1 \\ \hline
T & 3 & \cellcolor[HTML]{FE996B}2 & \cellcolor[HTML]{FE996B}2 & \cellcolor[HTML]{FE996B}2 & 3 & 4 & 4 & U, N, D & -2 \\ \hline
U & 4 & 3 & \cellcolor[HTML]{FE996B}2 & 3 & 3 & 4 & 5 & N & -2 \\ \hline
R & 5 & 4 & \cellcolor[HTML]{FE996B}3 & \cellcolor[HTML]{FE996B}3 & 4 & 4 & 5 & N, D & -3 \\ \hline
D & 6 & 5 & 4 & 4 & \cellcolor[HTML]{FE996B}3 & 4 & 5 & A & -3 \\ \hline
A & 7 & 6 & 5 & 5 & 4 & \cellcolor[HTML]{FE996B}3 & 4 & Y & -3 \\ \hline
Y & 8 & 7 & 6 & 6 & 5 & 4 & \cellcolor[HTML]{FE996B}3 & \textless{}/s\textgreater~& -3 \\ \hline
\end{tabular}\\ \vspace{.3cm}
\begin{tabular}{|
>{\columncolor[HTML]{32CB00}}l |l|l|l|l|l|l|l|}
\hline
\cellcolor[HTML]{C0C0C0} & \cellcolor[HTML]{C0C0C0} & \cellcolor[HTML]{C0C0C0}S & \cellcolor[HTML]{32CB00}U & \cellcolor[HTML]{32CB00}N & \cellcolor[HTML]{32CB00}D & \cellcolor[HTML]{32CB00}A & \cellcolor[HTML]{32CB00}Y \\ \hline
\cellcolor[HTML]{C0C0C0} & \cellcolor[HTML]{FE996B}0 & 1 & 2 & 3 & 4 & 5 & 6 \\ \hline
\cellcolor[HTML]{C0C0C0}S & 1 & \cellcolor[HTML]{FE996B}0 & 1 & 2 & 3 & 4 & 5 \\ \hline
\cellcolor[HTML]{C0C0C0}A & 2 & \cellcolor[HTML]{32CB00}1 & \cellcolor[HTML]{FE996B}1 & 2 & 3 & 3 & 4 \\ \hline
U &  &  & \cellcolor[HTML]{32CB00}1 &  &  &  &  \\ \hline
N &  &  &  & \cellcolor[HTML]{32CB00}1 &  &  &  \\ \hline
D &  &  &  &  & \cellcolor[HTML]{32CB00}1 &  &  \\ \hline
A &  &  &  &  &  & \cellcolor[HTML]{32CB00}1 &  \\ \hline
Y &  &  &  &  &  &  & \cellcolor[HTML]{32CB00}1 \\ \hline
\end{tabular}~~
\begin{tabular}{|l|l|l|l|l|l|l|l|}
\hline
\rowcolor[HTML]{32CB00} 
\cellcolor[HTML]{C0C0C0} & \cellcolor[HTML]{C0C0C0} & \cellcolor[HTML]{C0C0C0}S &  \cellcolor[HTML]{C0C0C0} U & N & D & A & Y \\ \hline
\cellcolor[HTML]{C0C0C0} & \cellcolor[HTML]{FE996B}0 & 1 & 2 & 3 & 4 & 5 & 6 \\ \hline
\cellcolor[HTML]{C0C0C0}S & 1 & \cellcolor[HTML]{FE996B}0 & 1 & 2 & 3 & 4 & 5 \\ \hline
\cellcolor[HTML]{C0C0C0}A & 2 & \cellcolor[HTML]{FE996B}1 & \cellcolor[HTML]{32CB00}1 & 2 & 3 & 3 & 4 \\ \hline
\cellcolor[HTML]{32CB00}N &  &  &  & \cellcolor[HTML]{32CB00}1 &  &  &  \\ \hline
\cellcolor[HTML]{32CB00}D &  &  &  &  & \cellcolor[HTML]{32CB00}1 &  &  \\ \hline
\cellcolor[HTML]{32CB00}A &  &  &  &  &  & \cellcolor[HTML]{32CB00}1 &  \\ \hline
\cellcolor[HTML]{32CB00}Y &  &  &  &  &  &  & \cellcolor[HTML]{32CB00}1 \\ \hline
\multicolumn{7}{c}{}\\
\end{tabular}
\end{table}

\section{Exposure bias}
\label{sec:exposure}
A key limitation of teacher forcing for sequence learning stems from the discrepancy between the training
and test objectives. One trains the model using conditional log-likelihood $\mathcal{O}_{\text{CLL}}$, but evaluates the quality of the model using
empirical reward $\mathcal{O}_{\text{ER}}$.

\begin{figure}[t]
\vspace*{-.2cm}
\begin{center}
\begin{tabular}{c@{\hspace*{2cm}}c}
      (a) {Teacher Forcing (MLE)} &
      (b) {Scheduled Sampling} \\[-.1cm]
      \includegraphics[width=0.35\textwidth]{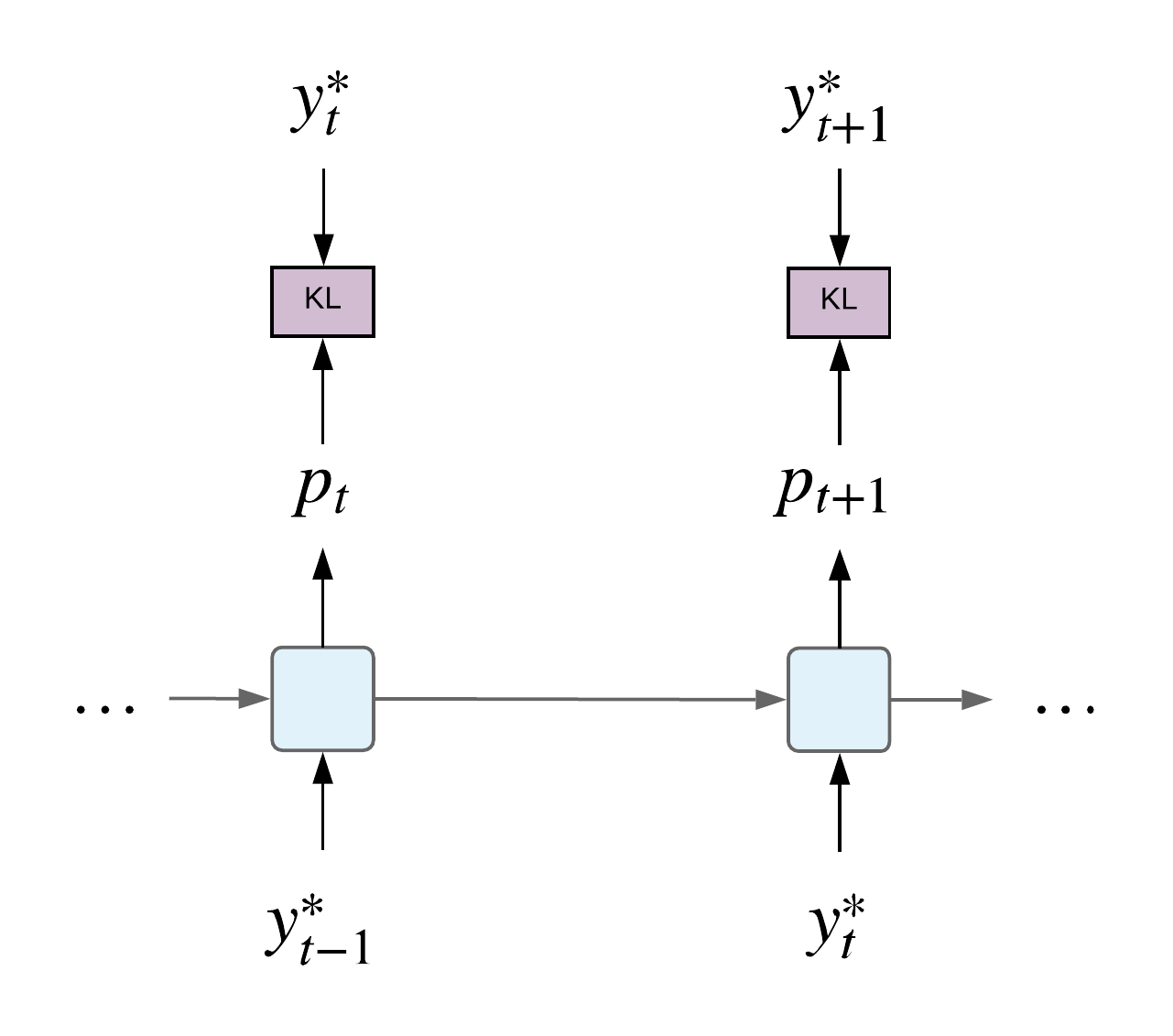} &
      \includegraphics[width=0.35\textwidth]{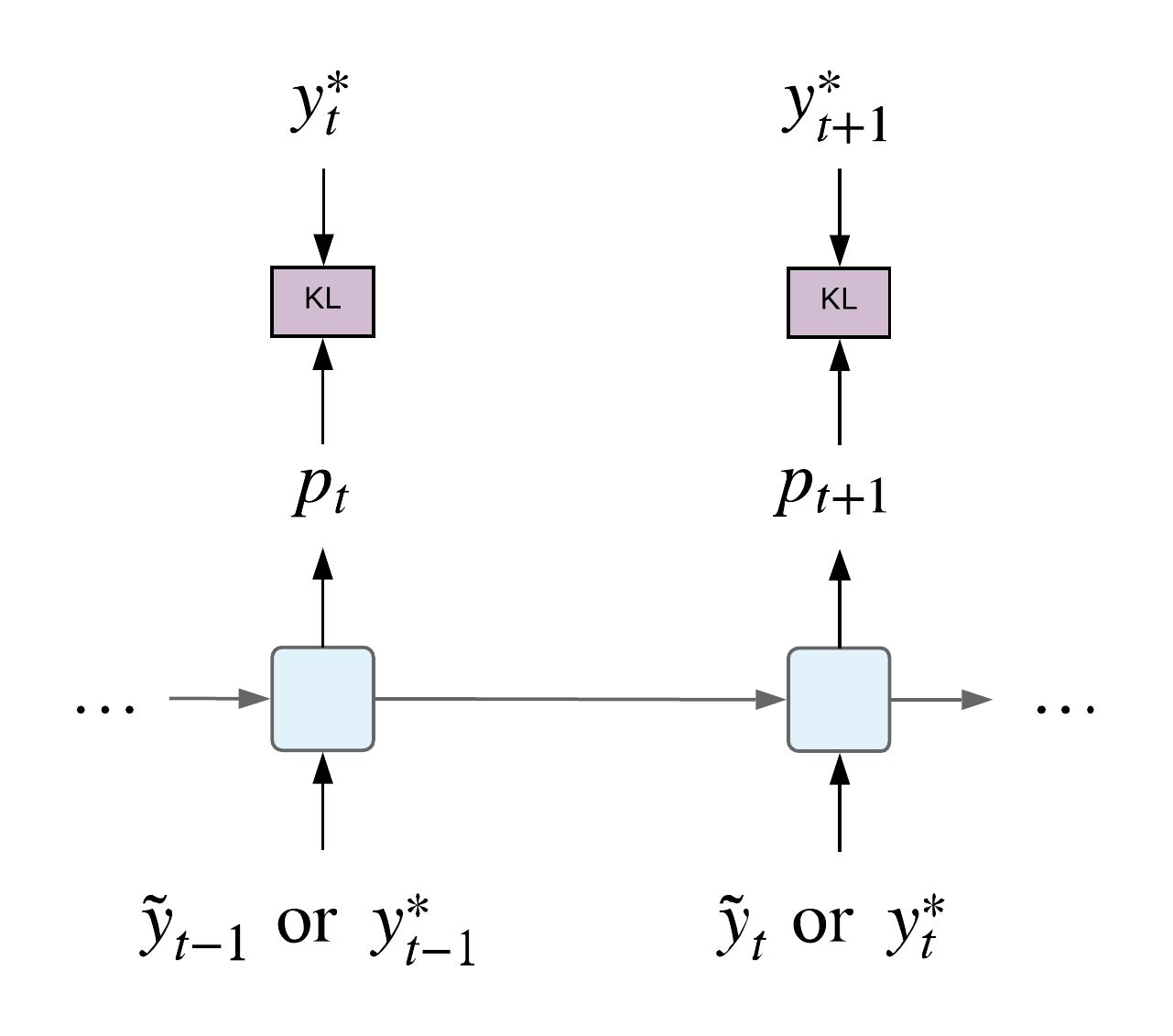} \\[-.1cm]
      \includegraphics[width=0.35\textwidth]{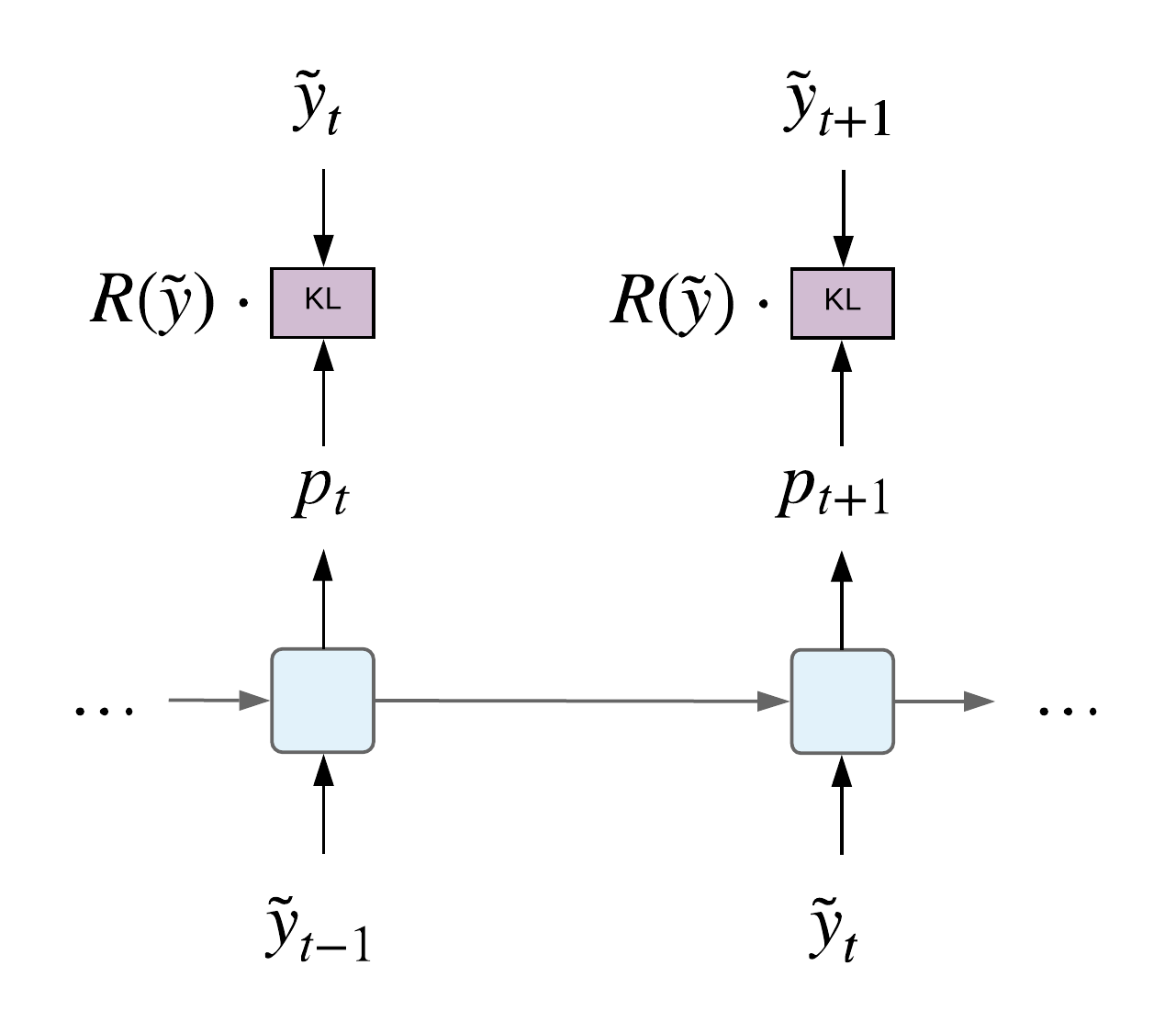} &
      \includegraphics[width=0.35\textwidth]{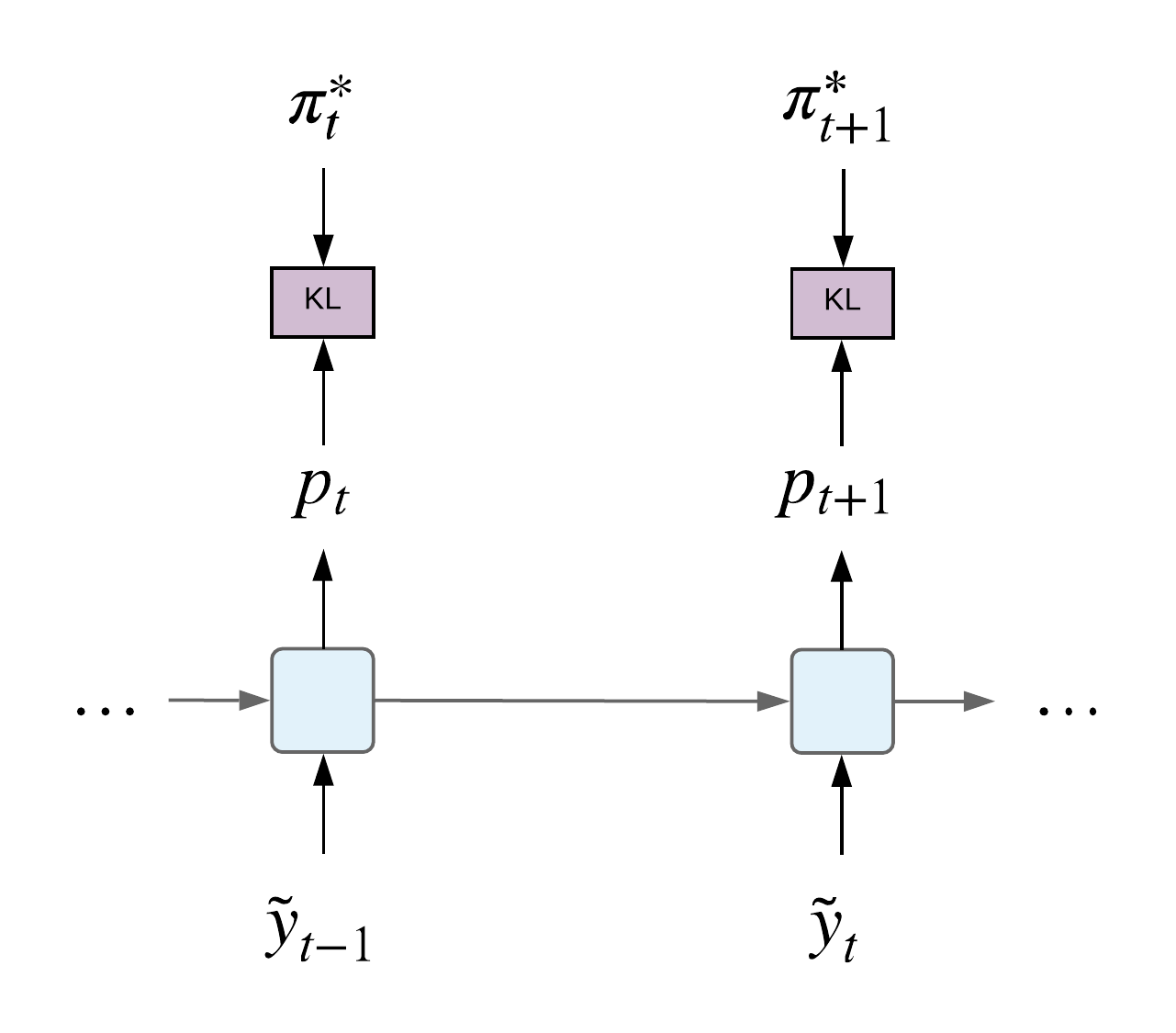} \\[-.1cm]
      (c) {Policy Gradient} &
      (d) {Optimal Completion Distillation}\\
\vspace*{-.2cm}
\end{tabular}
\end{center}
\comment{
                \end{tabular}
                \end{figure}
                \begin{figure}[t]
                  \centering
                  \begin{subfigure}[t]{0.5\textwidth}
                      \centering
                      \includegraphics[width=0.75\textwidth]{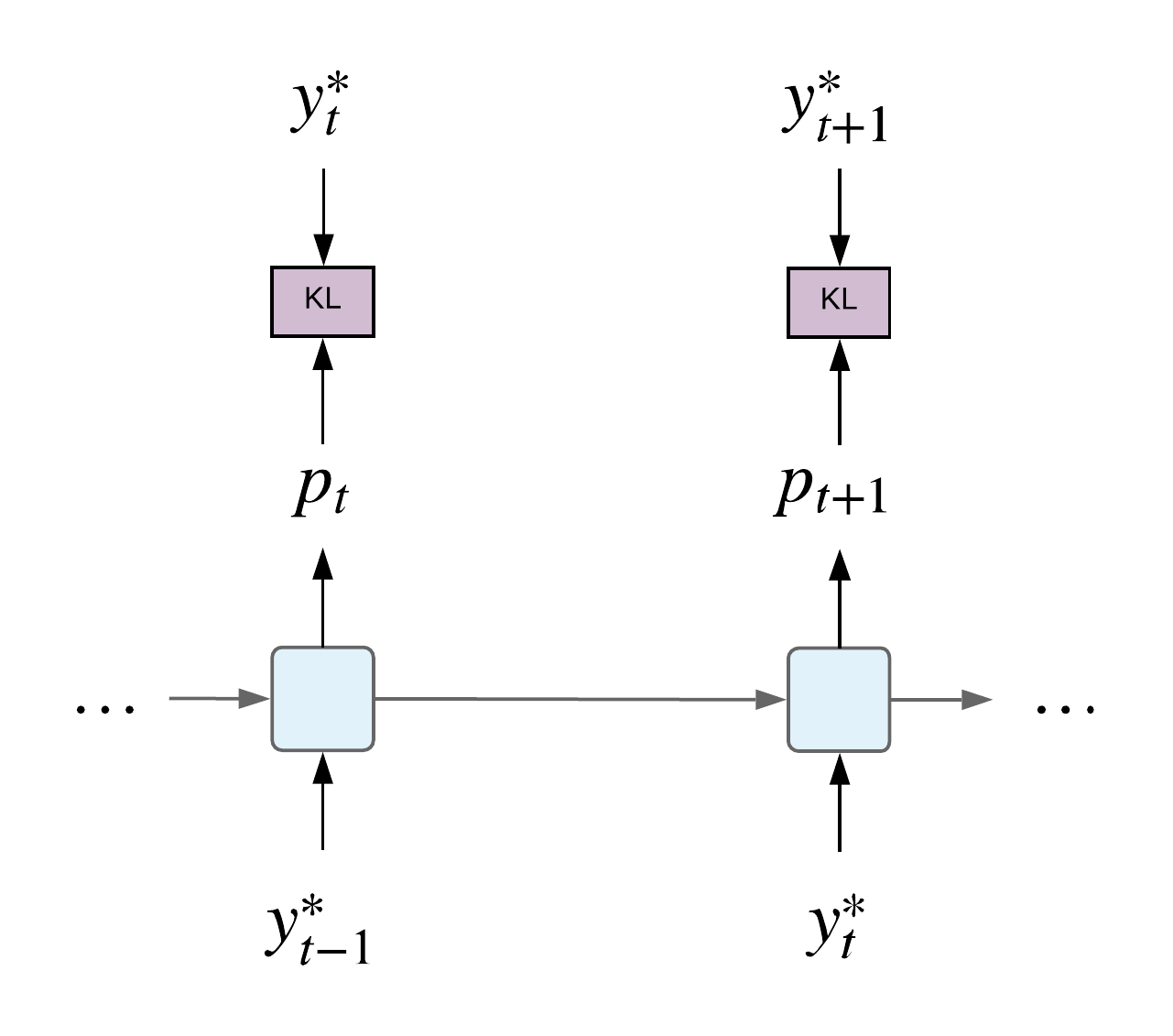}
                      \caption{Teacher Forcing}
                      \label{subfig:teacher-forcing}
                  \end{subfigure}%
                  \begin{subfigure}[t]{0.5\textwidth}
                      \centering
                      \includegraphics[width=0.75\textwidth]{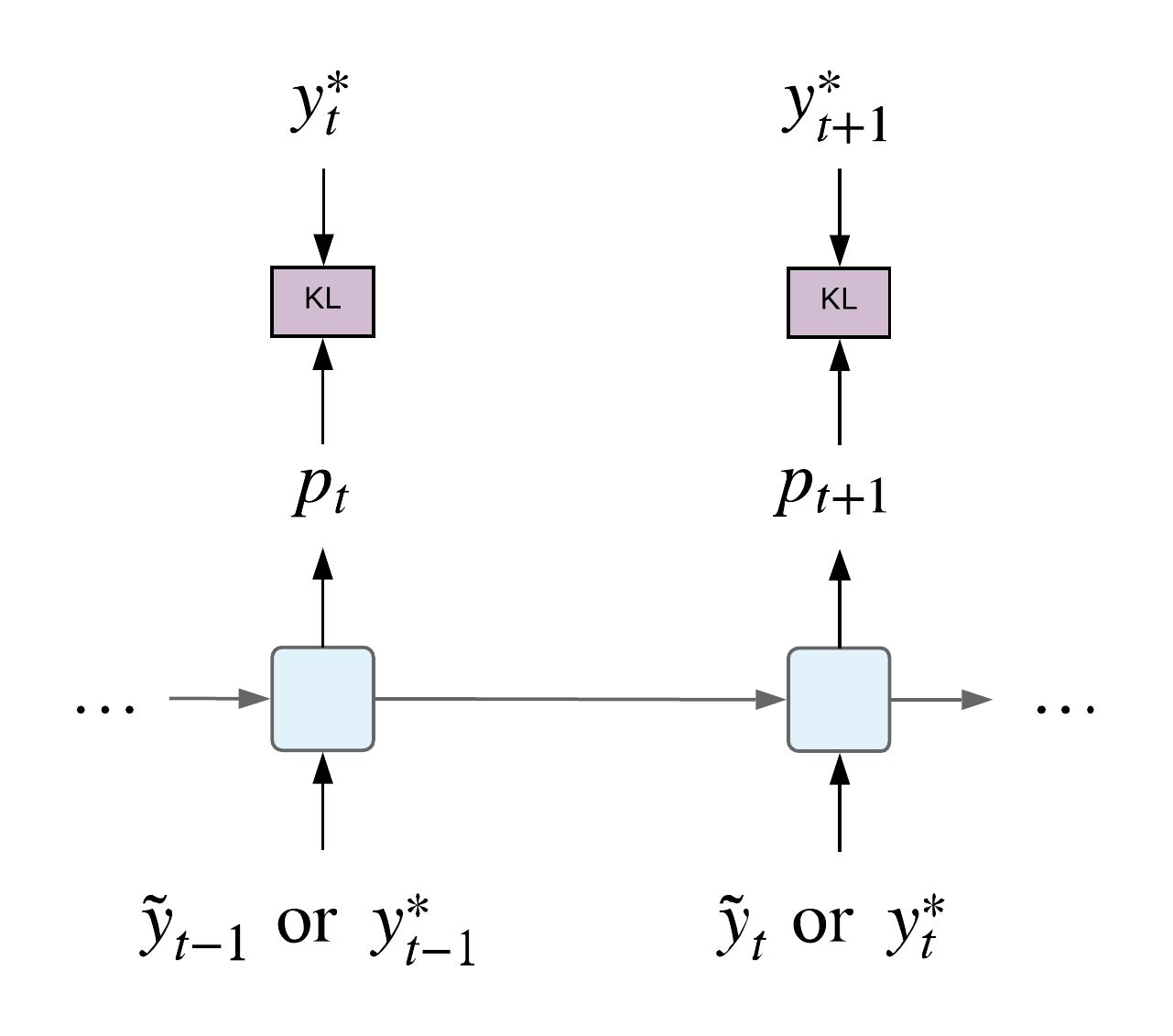}
                      \caption{Scheduled Sampling}
                      \label{subfig:ss}
                  \end{subfigure}
                  \begin{subfigure}[t]{0.5\textwidth}
                      \centering
                      \includegraphics[width=0.75\textwidth]{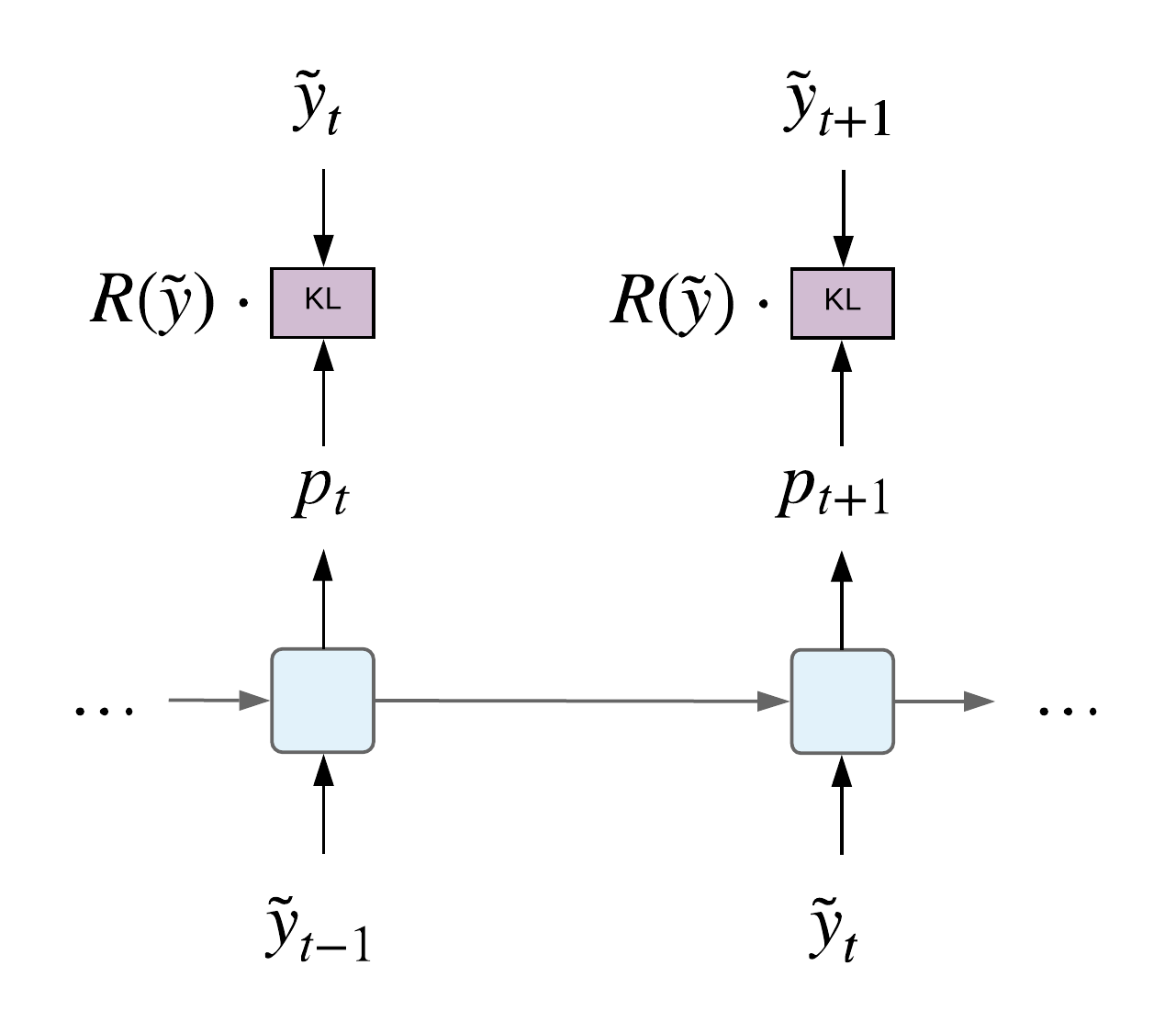}
                      \caption{Policy Gradient}
                      \label{subfig:policy}
                  \end{subfigure}%
                  \begin{subfigure}[t]{0.5\textwidth}
                      \centering
                      \includegraphics[width=0.75\textwidth]{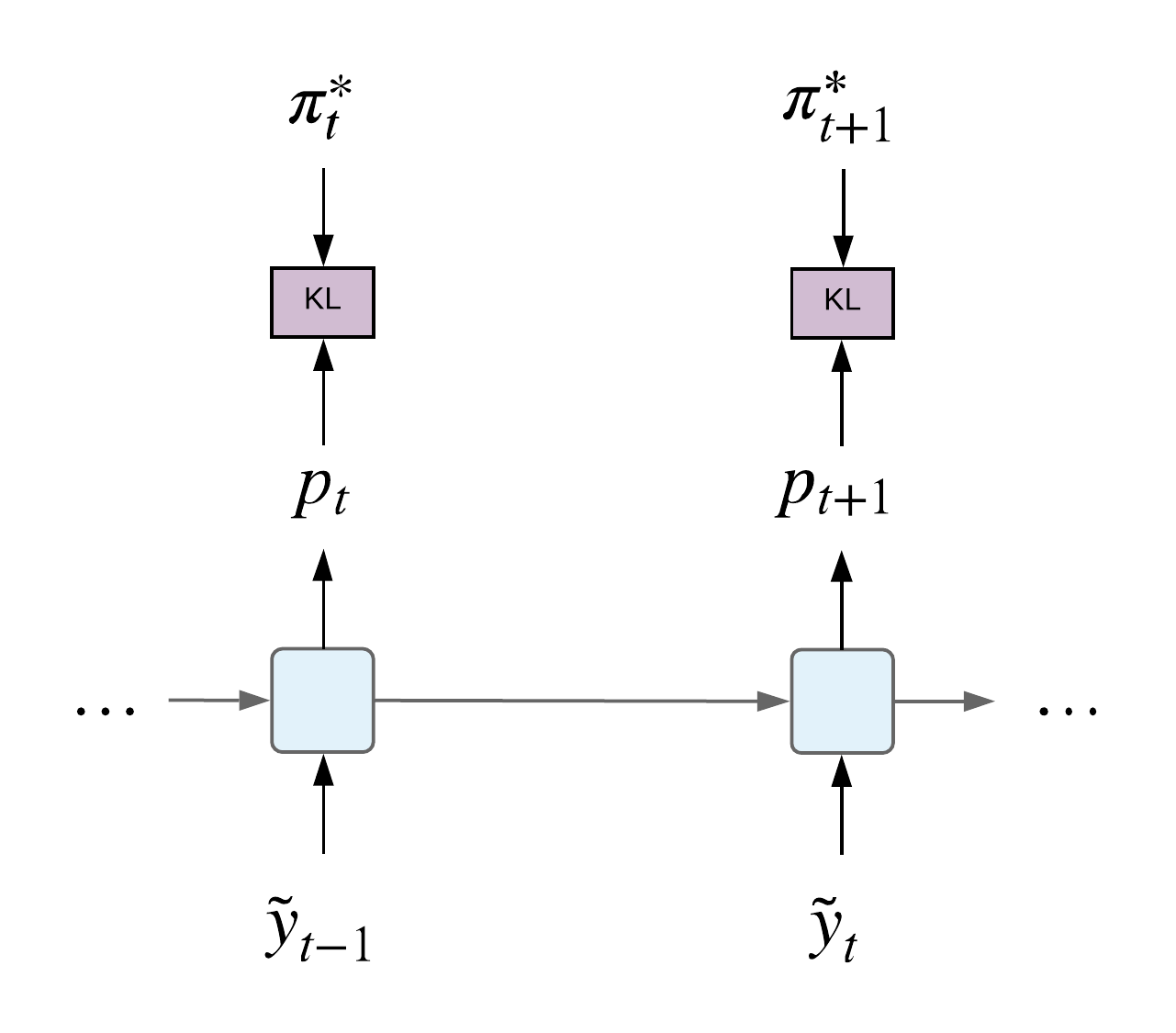}
                      \caption{Optimal Completion Distillation}
                      \label{subfig:ocd}
                  \end{subfigure}
                  \subref{subfig:teacher-forcing}
}
\caption{Illustration of different training strategies for autoregressive sequence models.
(a) Teacher Forcing: the model conditions on correct prefixes and is taught to predict the next ground truth token.
(b) Scheduled Sampling: the model conditions on tokens either from ground truth or drawn from the model and is taught to predict the next ground truth token regardless.
(c) Policy Gradient: the model conditions on prefixes drawn from the model and is encouraged to reinforce sequences with a large sequence reward $R(\tilde y)$.
(d) Optimal Completion Distillation: the model conditions on prefixes drawn from the model and is taught to predict an optimal completion policy $\pi^*$ specific to the prefix.}
\label{fig:trainingstrats}
\end{figure}

Unlike teacher forcing and Scheduled Sampling (SS), policy gradient approaches (\eg~\cite{ranzato-iclr-2016,bahdanau-iclr-2017})
and OCD aim to optimize the empirical reward objective \eqref{eq:reward} on the training set. We illustrate four different training strategies
of MLE, SS, Policy Gradient and OCD in \figref{fig:trainingstrats}. The drawback of policy gradient techniques is twofold: 1) they cannot
easily incorporate ground truth sequence information except through the reward function, and 2) they have difficulty reducing the variance of the gradients to
perform proper credit assignment. Accordingly, most policy gradient approaches~\cite{ranzato-iclr-2016,bahdanau-iclr-2017,wu2016google} pre-train the model
using teacher forcing. By contrast, the OCD method proposed in this paper defines an optimal completion policy $\pi^*_t$
for any {\em off-policy} prefix by incorporating the ground truth information. Then, OCD optimizes a token level log-loss
and alleviates the credit assignment problem. Finally, training is much more stable, and we do not require initialization nor joint optimization with MLE.

There is an intuitive notion of {\em exposure bias}~\cite{ranzato-iclr-2016} discussed in the literature as a limitation of teacher forcing.
We formalize this notion as follows. One can think of the optimization of the log loss~\eqref{eq:cll} in an autoregressive models as a classification problem, where the input to the classifier is a tuple $(\bs, \bys_{<t})$ and the correct output is $\a^*_i$, where $\bas_{<t} \equiv (y^*_1, \ldots, y^*_{t-1})$. Then the training dataset comprises different examples and different prefixes of the ground truth sequence. The key challenge is that once the model is trained, one should not expect the model to generalize to a new prefix $\by_{<t}$ that does not come from the training distribution of $P(\bys_{<t})$. This problem can
become severe as $\by_{<t}$ becomes more dissimilar to correct prefixes. During inference, when one conducts beam search with a large beam size then
one is more likely to discover wrong generalization of $\pit(\hat{y}_{t} | \bai_{<t}, \bx)$, because the sequence is optimized globally.
A natural strategy to remedy this issue is to train on arbitrary prefixes. Unlike the aforementioned techniques OCD can train on any prefix
given its {\em off-policy} nature.

Figure~\ref{fig:beam} illustrates how increasing the beam size for MLE and SS during inference decreases their performance on WSJ datasets to above $11\%$ WER. \acronym suffers a degradation in the performance too but it never gets above $10\%$ WER.
\begin{figure}[H]
\centering
  \includegraphics[width=.57\linewidth]{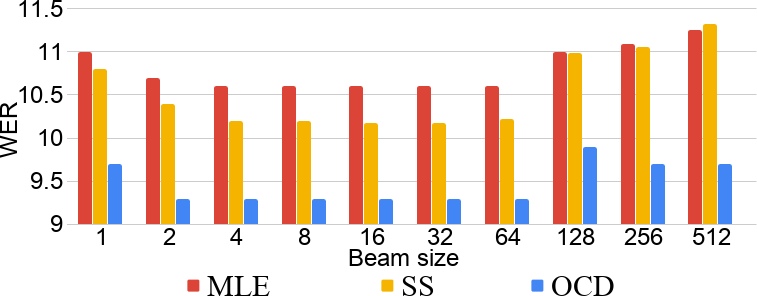}
\caption{Word Error Rate (WER) of WSJ with MLE, SS and OCD for different beam sizes.
}
\label{fig:beam}
\end{figure}

\section{Architecture}
\label{sec:architecture}
{\bf WSJ.} The input audio signal is converted into $80$-dimensional filterbank features computed every $10$ms with delta and delta-delta acceleration, normalized with per-speaker mean and variance generated by Kaldi \citep{povey2011kaldi}. Our encoder uses 2-layers of convolutions with $3 \times 3$ filters, stride $2 \times 2$ and $32$ channels, followed by a convolutional LSTM with 1D-convolution of filter width $3$, followed by 3 LSTM layers with $256$ cell size. We also apply batch-normalization between each layer in the encoder. The attention-based decoder is a 1-layer LSTM with $256$ cell size with content-based attention. We \rebut{use Xavier initializer~\citep{glorot2010understanding} and }train our models for 300 epochs of batch size 8 with 8 async workers. We separately tune the learning rate for our baseline and OCD model, $0.0007$ for OCD vs $0.001$ for baseline. We apply a single $0.01$ drop of learning rate when validation CER plateaus, the same as for our baseline. Both happen around 225 epoch. We implemented our experiments\footnote{We are in the process of releasing the code for \acronym.} in TensorFlow \citep{abadi2016tensorflow}.

{\bf Librispeech.} Since the dataset is larger than WSJ, we use a larger batch size of 16, smaller learning rate of $0.0005$ for baseline and $0.0003$ for \acronym. Models are trained for 70 epochs. We remove the convolutional LSTM layers of the encoder, increase the number of LSTM layers in the encoder to 6, and increase the LSTM cell size to 384. All other configs are the same as the WSJ setup.

\section{Hamming distance VS Edit Distance during training}
\label{sec:hamm}
Figure~\ref{fig:hamedit} plots the edit distance on training data of \acronym and MLE for fixed hamming distances during training. The plot shows that for a fixed Hamming distance (which is the metric that MLE correlates with more), \acronym achieves a lower edit distance compared to MLE. This gives evidence that \acronym is indeed optimizing for edit distance as intended.

\begin{figure}[h]
\centering
  \includegraphics[width=.5\linewidth]{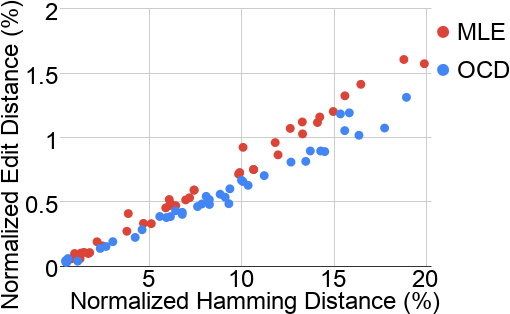}
\caption{WSJ training Character Error Rate (CER) of MLE and OCD over Character Accuracy at different checkpoints during training. 
}
\label{fig:hamedit}
\end{figure}

\comment{
\begin{table}[]
\comment{\begin{tabular}{|l|l|l|l|l|l|l|l|}
\hline
\rowcolor[HTML]{C0C0C0} 
\cellcolor[HTML]{FFFFFF}{\color[HTML]{000000} } &  & S & U & N & D & A & Y \\ \hline
\cellcolor[HTML]{C0C0C0} & 0 & 1 & 2 & 3 & 4 & 5 & 6 \\ \hline
\cellcolor[HTML]{C0C0C0}S & 1 & 0 & 1 & 2 & 3 & 4 & 5 \\ \hline
\cellcolor[HTML]{C0C0C0}A & 2 & 1 & 1 & 2 & 3 & 3 & 4 \\ \hline
\cellcolor[HTML]{C0C0C0}T & 3 & \cellcolor[HTML]{FE996B}2 & \cellcolor[HTML]{FE996B}2 & \cellcolor[HTML]{FE996B}2 & 3 & 4 & 4 \\ \hline
? &  &  &  &  &  &  &  \\ \hline
? &  &  &  &  &  &  &  \\ \hline
? &  &  &  &  &  &  &  \\ \hline
? &  &  &  &  &  &  &  \\ \hline
? &  &  &  &  &  &  &  \\ \hline
\end{tabular}~}\begin{tabular}{|
>{\columncolor[HTML]{32CB00}}l |l|l|l|l|l|l|l|}
\hline
\cellcolor[HTML]{FFFFFF}{\color[HTML]{000000} } & \cellcolor[HTML]{C0C0C0} & \cellcolor[HTML]{C0C0C0}S & \cellcolor[HTML]{32CB00}U & \cellcolor[HTML]{32CB00}N & \cellcolor[HTML]{32CB00}D & \cellcolor[HTML]{32CB00}A & \cellcolor[HTML]{32CB00}Y \\ \hline
\cellcolor[HTML]{C0C0C0} & 0 & 1 & 2 & 3 & 4 & 5 & 6 \\ \hline
\cellcolor[HTML]{C0C0C0}S & 1 & 0 & 1 & 2 & 3 & 4 & 5 \\ \hline
\cellcolor[HTML]{C0C0C0}A & 2 & 1 & 1 & 2 & 3 & 3 & 4 \\ \hline
\cellcolor[HTML]{C0C0C0}T & 3 & \cellcolor[HTML]{32CB00}2 & \cellcolor[HTML]{FE996B}2 & \cellcolor[HTML]{FE996B}2 & 3 & 4 & 4 \\ \hline
U &  &  & \cellcolor[HTML]{32CB00}2 &  &  &  &  \\ \hline
N &  &  &  & \cellcolor[HTML]{32CB00}2 &  &  &  \\ \hline
D &  &  &  &  & \cellcolor[HTML]{32CB00}2 &  &  \\ \hline
A &  &  &  &  &  & \cellcolor[HTML]{32CB00}2 & \cellcolor[HTML]{FFFFFF} \\ \hline
Y &  &  &  &  &  &  & \cellcolor[HTML]{32CB00}2 \\ \hline
\end{tabular}~\begin{tabular}{|l|l|l|l|l|l|l|l|}
\hline
\rowcolor[HTML]{32CB00} 
\cellcolor[HTML]{FFFFFF}{\color[HTML]{000000} } & \cellcolor[HTML]{C0C0C0} & \cellcolor[HTML]{C0C0C0}S & \cellcolor[HTML]{C0C0C0}U & N & D & A & Y \\ \hline
\cellcolor[HTML]{C0C0C0} & 0 & 1 & 2 & 3 & 4 & 5 & 6 \\ \hline
\cellcolor[HTML]{C0C0C0}S & 1 & 0 & 1 & 2 & 3 & 4 & 5 \\ \hline
\cellcolor[HTML]{C0C0C0}A & 2 & 1 & 1 & 2 & 3 & 3 & 4 \\ \hline
\cellcolor[HTML]{C0C0C0}T & 3 & \cellcolor[HTML]{FE996B}2 & \cellcolor[HTML]{32CB00}2 & \cellcolor[HTML]{FE996B}2 & 3 & 4 & 4 \\ \hline
\cellcolor[HTML]{32CB00}N &  &  &  & \cellcolor[HTML]{32CB00}2 &  &  &  \\ \hline
\cellcolor[HTML]{32CB00}D &  &  &  &  & \cellcolor[HTML]{32CB00}2 &  &  \\ \hline
\cellcolor[HTML]{32CB00}A &  &  &  &  &  & \cellcolor[HTML]{32CB00}2 & \cellcolor[HTML]{FFFFFF} \\ \hline
\cellcolor[HTML]{32CB00}Y &  &  &  &  &  &  & \cellcolor[HTML]{32CB00}2 \\\hline
\cellcolor[HTML]{C0C0C0} &  &  &  &  &  &  &  \\ \hline
\end{tabular}~\begin{tabular}{|
>{\columncolor[HTML]{C0C0C0}}l |l|l|l|l|l|l|l|}
\hline
\cellcolor[HTML]{FFFFFF}{\color[HTML]{000000} } & \cellcolor[HTML]{C0C0C0} & \cellcolor[HTML]{C0C0C0}S & \cellcolor[HTML]{C0C0C0}U & \cellcolor[HTML]{C0C0C0}N & \cellcolor[HTML]{32CB00}D & \cellcolor[HTML]{32CB00}A & \cellcolor[HTML]{32CB00}Y \\ \hline
 & 0 & 1 & 2 & 3 & 4 & 5 & 6 \\ \hline
S & 1 & 0 & 1 & 2 & 3 & 4 & 5 \\ \hline
A & 2 & 1 & 1 & 2 & 3 & 3 & 4 \\ \hline
T & 3 & \cellcolor[HTML]{FE996B}2 & \cellcolor[HTML]{FE996B}2 & \cellcolor[HTML]{32CB00}2 & 3 & 4 & 4 \\ \hline
\cellcolor[HTML]{32CB00}D &  &  &  &  & \cellcolor[HTML]{32CB00}2 &  &  \\ \hline
\cellcolor[HTML]{32CB00}A &  &  &  &  &  & \cellcolor[HTML]{32CB00}2 & \cellcolor[HTML]{FFFFFF} \\ \hline
\cellcolor[HTML]{32CB00}Y &  &  &  &  &  &  & \cellcolor[HTML]{32CB00}2 \\ \hline
 &  &  &  &  &  &  &  \\ \hline
 &  &  &  &  &  &  &  \\ \hline
\end{tabular}
\end{table}
\begin{table}[]

\end{table}}

%% file: ocd-iclr.bbl
\begin{thebibliography}{66}
\providecommand{\natexlab}[1]{#1}
\providecommand{\url}[1]{\texttt{#1}}
\expandafter\ifx\csname urlstyle\endcsname\relax
  \providecommand{\doi}[1]{doi: #1}\else
  \providecommand{\doi}{doi: \begingroup \urlstyle{rm}\Url}\fi

\bibitem[Abadi et~al.(2016)Abadi, Barham, Chen, Chen, Davis, Dean, Devin,
  Ghemawat, Irving, Isard, et~al.]{abadi2016tensorflow}
Mart{\'\i}n Abadi, Paul Barham, Jianmin Chen, Zhifeng Chen, Andy Davis, Jeffrey
  Dean, Matthieu Devin, Sanjay Ghemawat, Geoffrey Irving, Michael Isard, et~al.
\newblock Tensorflow: a system for large-scale machine learning.
\newblock In \emph{OSDI}, volume~16, pages 265--283, 2016.

\bibitem[Bahdanau et~al.(2015)Bahdanau, Cho, and Bengio]{bahdanau2014neural}
Dzmitry Bahdanau, Kyunghyun Cho, and Yoshua Bengio.
\newblock Neural machine translation by jointly learning to align and
  translate.
\newblock \emph{ICLR}, 2015.

\bibitem[Bahdanau et~al.(2016{\natexlab{a}})Bahdanau, Chorowski, Serdyuk,
  Brakel, and Bengio]{bahdanau-icassp-2016}
Dzmitry Bahdanau, Jan Chorowski, Dmitriy Serdyuk, Philemon Brakel, and Yoshua
  Bengio.
\newblock {End-to-End Attention-based Large Vocabulary Speech Recognition}.
\newblock \emph{{ICASSP}}, 2016{\natexlab{a}}.

\bibitem[Bahdanau et~al.(2016{\natexlab{b}})Bahdanau, Serdyuk, Brakel, Ke,
  Chorowski, Courville, and Bengio]{bahdanau-iclr-2016}
Dzmitry Bahdanau, Dmitriy Serdyuk, Philemon Brakel, Nan~Rosemary Ke, Jan
  Chorowski, Aaron Courville, and Yoshua Bengio.
\newblock {Task Loss Estimation for Sequence Prediction}.
\newblock \emph{ICLR Workshop}, 2016{\natexlab{b}}.

\bibitem[Bahdanau et~al.(2017)Bahdanau, Brakel, Xu, Goyal, Lowe, Pineau,
  Courville, and Bengio]{bahdanau-iclr-2017}
Dzmitry Bahdanau, Philemon Brakel, Kelvin Xu, Anirudh Goyal, Ryan Lowe, Joelle
  Pineau, Aaron Courville, and Yoshua Bengio.
\newblock {An Actor-Critic Algorithm for Sequence Prediction}.
\newblock \emph{{ICLR}}, 2017.

\bibitem[Bengio et~al.(2015)Bengio, Vinyals, Jaitly, and
  Shazeer]{bengio-nips-2015}
Samy Bengio, Oriol Vinyals, Navdeep Jaitly, and Noam~M. Shazeer.
\newblock {Scheduled Sampling for Sequence Prediction with Recurrent Neural
  Networks}.
\newblock \emph{NIPS}, 2015.

\bibitem[Britz et~al.(2017)Britz, Goldie, Luong, and Le]{britz2017massive}
Denny Britz, Anna Goldie, Minh-Thang Luong, and Quoc Le.
\newblock Massive exploration of neural machine translation architectures.
\newblock \emph{arXiv:1703.03906}, 2017.

\bibitem[Chan et~al.(2016)Chan, Jaitly, Le, and Vinyals]{chan-icassp-2016}
William Chan, Navdeep Jaitly, Quoc Le, and Oriol Vinyals.
\newblock {Listen, Attend and Spell: A Neural Network for Large Vocabulary
  Conversational Speech Recognition}.
\newblock \emph{{ICASSP}}, 2016.

\bibitem[Chan et~al.(2017)Chan, Zhang, and Jaitly]{chan-iclr-2017}
William Chan, Yu~Zhang, and Navdeep Jaitly.
\newblock {Latent Sequence Decompositions}.
\newblock \emph{{ICLR}}, 2017.

\bibitem[Chang et~al.(2015)Chang, Krishnamurthy, Agarwal, III, and
  Langford]{chang2015learning}
Kai-Wei Chang, Akshay Krishnamurthy, Alekh Agarwal, Hal~Daumé III, and John
  Langford.
\newblock {Learning to Search Better Than Your Teacher}.
\newblock In \emph{{ICML}}, 2015.

\bibitem[Cheng and Boots(2018)]{cheng2018convergence}
Ching-An Cheng and Byron Boots.
\newblock Convergence of value aggregation for imitation learning.
\newblock \emph{arXiv preprint arXiv:1801.07292}, 2018.

\bibitem[Chiu et~al.(2017)Chiu, Sainath, Wu, Prabhavalkar, Nguyen, Chen,
  Kannan, Weiss, Rao, Gonina, et~al.]{chiu2017state}
Chung-Cheng Chiu, Tara~N Sainath, Yonghui Wu, Rohit Prabhavalkar, Patrick
  Nguyen, Zhifeng Chen, Anjuli Kannan, Ron~J Weiss, Kanishka Rao, Katya Gonina,
  et~al.
\newblock State-of-the-art speech recognition with sequence-to-sequence models.
\newblock \emph{arXiv:1712.01769}, 2017.

\bibitem[Cho et~al.(2014)Cho, Van~Merri{\"e}nboer, Gulcehre, Bahdanau,
  Bougares, Schwenk, and Bengio]{cho2014learning}
Kyunghyun Cho, Bart Van~Merri{\"e}nboer, Caglar Gulcehre, Dzmitry Bahdanau,
  Fethi Bougares, Holger Schwenk, and Yoshua Bengio.
\newblock Learning phrase representations using rnn encoder-decoder for
  statistical machine translation.
\newblock \emph{EMNLP}, 2014.

\bibitem[Chorowski and Jaitly(2017)]{chorowski-interspeech-2017}
Jan Chorowski and Navdeep Jaitly.
\newblock {Towards better decoding and language model integration in sequence
  to sequence models}.
\newblock \emph{{INTERSPEECH}}, 2017.

\bibitem[Collobert et~al.(2016)Collobert, Puhrsch, and
  Synnaeve]{collobert2016wav2letter}
Ronan Collobert, Christian Puhrsch, and Gabriel Synnaeve.
\newblock Wav2letter: an end-to-end convnet-based speech recognition system.
\newblock \emph{arXiv preprint arXiv:1609.03193}, 2016.

\bibitem[Daum{\'e} et~al.(2009)Daum{\'e}, Langford, and Marcu]{daumeetal09}
Hal Daum{\'e}, III, John Langford, and Daniel Marcu.
\newblock Search-based structured prediction.
\newblock \emph{Mach. Learn. J.}, 2009.

\bibitem[Daum{\'e}~III and Marcu(2005)]{l2s2005}
Hal Daum{\'e}~III and Daniel Marcu.
\newblock Learning as search optimization: Approximate large margin methods for
  structured prediction.
\newblock \emph{ICML}, 2005.

\bibitem[Ding and Soricut(2017)]{ding2017cold}
Nan Ding and Radu Soricut.
\newblock \rebut{Cold-Start Reinforcement Learning with Softmax Policy
  Gradient}.
\newblock In \emph{Advances in Neural Information Processing Systems}, pages
  2817--2826, 2017.

\bibitem[Edunov et~al.(2018)Edunov, Ott, Auli, Grangier, and
  Ranzato]{edunov2017classical}
Sergey Edunov, Myle Ott, Michael Auli, David Grangier, and Marc'Aurelio
  Ranzato.
\newblock Classical structured prediction losses for sequence to sequence
  learning.
\newblock \emph{NAACL}, 2018.

\bibitem[Elbayad et~al.(2018)Elbayad, Besacier, and Verbeek]{elbayad2018token}
Maha Elbayad, Laurent Besacier, and Jakob Verbeek.
\newblock Token-level and sequence-level loss smoothing for rnn language
  models.
\newblock \emph{ACL}, 2018.

\bibitem[Gehring et~al.(2017)Gehring, Auli, Grangier, Yarats, and
  Dauphin]{gehring-arxiv-2017}
Jonas Gehring, Michael Auli, David Grangier, Denis Yarats, and Yann~N. Dauphin.
\newblock {Convolutional Sequence to Sequence Learning}.
\newblock In \emph{arXiv:1705.03122}, 2017.

\bibitem[Glorot and Bengio(2010)]{glorot2010understanding}
Xavier Glorot and Yoshua Bengio.
\newblock \rebut{Understanding the difficulty of training deep feedforward
  neural networks}.
\newblock In \emph{Proceedings of the thirteenth international conference on
  artificial intelligence and statistics}, pages 249--256, 2010.

\bibitem[Goodman et~al.(2016)Goodman, Vlachos, and
  Naradowsky]{goodman2016noise}
James Goodman, Andreas Vlachos, and Jason Naradowsky.
\newblock Noise reduction and targeted exploration in imitation learning for
  abstract meaning representation parsing.
\newblock In \emph{Proceedings of the 54th Annual Meeting of the Association
  for Computational Linguistics (Volume 1: Long Papers)}, volume~1, pages
  1--11, 2016.

\bibitem[Graves and Jaitly(2014)]{graves-icml-2014}
Alex Graves and Navdeep Jaitly.
\newblock {Towards End-to-End Speech Recognition with Recurrent Neural
  Networks}.
\newblock \emph{{ICML}}, 2014.

\bibitem[Hassan et~al.(2018)Hassan, Aue, Chen, Chowdhary, Clark, Federmann,
  Huang, Junczys-Dowmunt, Lewis, Li, et~al.]{parity2018microsoft}
Hany Hassan, Anthony Aue, Chang Chen, Vishal Chowdhary, Jonathan Clark,
  Christian Federmann, Xuedong Huang, Marcin Junczys-Dowmunt, William Lewis,
  Mu~Li, et~al.
\newblock Achieving human parity on automatic chinese to english news
  translation.
\newblock \emph{arXiv:1803.05567}, 2018.

\bibitem[Hinton et~al.(2014)Hinton, Vinyals, and Dean]{hinton-nips-2014}
Geoffrey Hinton, Oriol Vinyals, and Jeff Dean.
\newblock {Distilling the Knowledge in a Neural Network}.
\newblock In \emph{{Neural Information Processing Systems: Workshop Deep
  Learning and Representation Learning Workshop}}, 2014.

\bibitem[Hochreiter and Schmidhuber(1997)]{hochreiter-nc-1997}
Sepp Hochreiter and Jurgen Schmidhuber.
\newblock {Long Short-Term Memory}.
\newblock \emph{Neural Comput.}, 9, 1997.

\bibitem[Karita et~al.(2018)Karita, Ogawa, Delcroix, and
  Nakatani]{karita-icassp-2018}
Shigeki Karita, Atsunori Ogawa, Marc Delcroix, and Tomohiro Nakatani.
\newblock \rebut{Sequence Training of Encoder-decoder Model Using Policy
  Gradient for End- To-end Speech Recognition}.
\newblock In \emph{{ICASSP}}, 2018.

\bibitem[Kim et~al.(2017)Kim, Hori, and Watanabe]{kim-icassp-2017}
Suyoun Kim, Takaaki Hori, and Shinji Watanabe.
\newblock {Joint CTC-Attention based End-to-End Speech Recognition using
  Multi-task Learning}.
\newblock \emph{{ICASSP}}, 2017.

\bibitem[Koehn et~al.(2007)Koehn, Hoang, Birch, Callison-Burch, Federico,
  Bertoldi, Cowan, Shen, Moran, Zens, et~al.]{koehn2007moses}
Philipp Koehn, Hieu Hoang, Alexandra Birch, Chris Callison-Burch, Marcello
  Federico, Nicola Bertoldi, Brooke Cowan, Wade Shen, Christine Moran, Richard
  Zens, et~al.
\newblock Moses: Open source toolkit for statistical machine translation.
\newblock \emph{Proceedings of the 45th annual meeting of the ACL on
  interactive poster and demonstration sessions}, 2007.

\bibitem[Leblond et~al.(2018)Leblond, Alayrac, Osokin, and
  Lacoste-Julien]{searnn2017}
R{\'e}mi Leblond, Jean-Baptiste Alayrac, Anton Osokin, and Simon
  Lacoste-Julien.
\newblock {SEARNN: Training RNNs with global-local losses}.
\newblock \emph{{ICLR}}, 2018.

\bibitem[Levenshtein(1966)]{levenshtein1966binary}
Vladimir~I Levenshtein.
\newblock {Binary codes capable of correcting deletions, insertions, and
  reversals}.
\newblock \emph{Soviet Physics Doklady}, 10\penalty0 (8):\penalty0 707--710,
  1966.

\bibitem[Liang et~al.(2018)Liang, Huang, and Lipton]{liang2018learning}
Davis Liang, Zhiheng Huang, and Zachary~C Lipton.
\newblock Learning noise-invariant representations for robust speech
  recognition.
\newblock \emph{arXiv preprint arXiv:1807.06610}, 2018.

\bibitem[Liptchinsky et~al.(2017)Liptchinsky, Synnaeve, and
  Collobert]{liptchinsky-arxiv-2017}
Vitaliy Liptchinsky, Gabriel Synnaeve, and Ronan Collobert.
\newblock {Letter-Based Speech Recognition with Gated ConvNets}.
\newblock In \emph{{arXiv:1712.09444}}, 2017.

\bibitem[Liu et~al.(2017)Liu, Zhu, Li, and Satheesh]{liu-icml-2017}
Hairong Liu, Zhenyao Zhu, Xiangang Li, and Sanjeev Satheesh.
\newblock {Gram-CTC: Automatic Unit Selection and Target Decomposition for
  Sequence Labelling}.
\newblock \emph{{ICML}}, 2017.

\bibitem[Ma et~al.(2017)Ma, Yin, Liu, Neubig, and Hovy]{ma2017softmax}
Xuezhe Ma, Pengcheng Yin, Jingzhou Liu, Graham Neubig, and Eduard Hovy.
\newblock Softmax q-distribution estimation for structured prediction: A
  theoretical interpretation for raml.
\newblock \emph{arXiv:1705.07136}, 2017.

\bibitem[Mnih et~al.(2015)Mnih, Kavukcuoglu, Silver, Rusu, Veness, Bellemare,
  Graves, Riedmiller, Fidjeland, Ostrovski, Petersen, Beattie, Sadik,
  Antonoglou, King, Kumaran, Wierstra, Legg, and Hassabis]{mnihetal15}
Volodymyr Mnih, Koray Kavukcuoglu, David Silver, Andrei~A. Rusu, Joel Veness,
  Marc~G. Bellemare, Alex Graves, Martin Riedmiller, Andreas~K. Fidjeland,
  Georg Ostrovski, Stig Petersen, Charles Beattie, Amir Sadik, Ioannis
  Antonoglou, Helen King, Dharshan Kumaran, Daan Wierstra, Shane Legg, and
  Demis Hassabis.
\newblock Human-level control through deep reinforcement learning.
\newblock \emph{Nature}, 2015.

\bibitem[Norouzi et~al.(2016)Norouzi, Bengio, Chen, Jaitly, Schuster, Wu, and
  Schuurmans]{norouzi-nips-2016}
Mohammad Norouzi, Samy Bengio, Zhifeng Chen, Navdeep Jaitly, Mike Schuster,
  Yonghui Wu, and Dale Schuurmans.
\newblock {Reward Augmented Maximum Likelihood for Neural Structured
  Prediction}.
\newblock \emph{{NIPS}}, 2016.

\bibitem[Panayotov et~al.(2015)Panayotov, Chen, Povey, and
  Khudanpur]{librispeech}
Vassil Panayotov, Guoguo Chen, Daniel Povey, and Sanjeev Khudanpur.
\newblock Librispeech: an asr corpus based on public domain audio books.
\newblock In \emph{Acoustics, Speech and Signal Processing (ICASSP), 2015 IEEE
  International Conference on}, pages 5206--5210. IEEE, 2015.

\bibitem[Paul and Baker(1992)]{paul1992design}
Douglas~B Paul and Janet~M Baker.
\newblock The design for the wall street journal-based csr corpus.
\newblock In \emph{Proceedings of the workshop on Speech and Natural Language},
  pages 357--362. Association for Computational Linguistics, 1992.

\bibitem[Pereyra et~al.(2017)Pereyra, Tucker, Chorowski, Łukasz Kaiser, and
  Hinton]{pereyra-iclr-2017}
Gabriel Pereyra, George Tucker, Jan Chorowski, Łukasz Kaiser, and Geoffrey
  Hinton.
\newblock {Regularizing Neural Networks by Penalizing Confident Output
  Distributions}.
\newblock \emph{ICLR Workshop}, 2017.

\bibitem[Povey et~al.(2011)Povey, Ghoshal, Boulianne, et~al.]{povey2011kaldi}
Daniel Povey, Arnab Ghoshal, Gilles Boulianne, et~al.
\newblock {The Kaldi Speech Recognition Toolkit}.
\newblock \emph{ASRU}, 2011.

\bibitem[Prabhavalkar et~al.(2018)Prabhavalkar, Sainath, Wu, Nguyen, Chen,
  Chiu, and Kannan]{prabhavalkar-icassp-2018}
Rohit Prabhavalkar, Tara~N. Sainath, Yonghui Wu, Patrick Nguyen, Zhifeng Chen,
  Chung-Cheng Chiu, and Anjuli Kannan.
\newblock {Minimum Word Error Rate Training for Attention-based
  Sequence-to-Sequence Models}.
\newblock In \emph{{ICASSP}}, 2018.

\bibitem[Ranzato et~al.(2016)Ranzato, Chopra, Auli, and
  Zaremba]{ranzato-iclr-2016}
Marc'Aurelio Ranzato, Sumit Chopra, Michael Auli, and Wojciech Zaremba.
\newblock {Sequence Level Training with Recurrent Neural Networks}.
\newblock \emph{ICLR}, 2016.

\bibitem[Rennie et~al.(2017)Rennie, Marcheret, Mroueh, Ross, and
  Goel]{rennie-cvpr-2017}
Steven~J. Rennie, Etienne Marcheret, Youssef Mroueh, Jarret Ross, and Vaibhava
  Goel.
\newblock {Self-critical Sequence Training for Image Captioning}.
\newblock In \emph{{CVPR}}, 2017.

\bibitem[Ross and Bagnell(2014)]{ross2014reinforcement}
Stephane Ross and J~Andrew Bagnell.
\newblock Reinforcement and imitation learning via interactive no-regret
  learning.
\newblock \emph{arXiv preprint arXiv:1406.5979}, 2014.

\bibitem[Ross et~al.(2011)Ross, Gordon, and Bagnell]{ross-aistats-2011}
Stephane Ross, Geoffrey~J Gordon, and J~Andrew Bagnell.
\newblock {A Reduction of Imitation Learning and Structured Prediction to
  No-Regret Online Learning}.
\newblock \emph{AISTATS}, 2011.

\bibitem[Rush et~al.(2015)Rush, Chopra, and Weston]{rush-emnlp-2015}
Alexander~M. Rush, Sumit Chopra, and Jason Weston.
\newblock {A Neural Attention Model for Abstractive Sentence Summarization}.
\newblock In \emph{{EMNLP}}, 2015.

\bibitem[Rusu et~al.(2016)Rusu, Colmenarejo, Gulcehre, Desjardins, Kirkpatrick,
  Pascanu, Mnih, Kavukcuoglu, and Hadsell]{rusu-iclr-2016}
Andrei~A. Rusu, Sergio~Gomez Colmenarejo, Caglar Gulcehre, Guillaume
  Desjardins, James Kirkpatrick, Razvan Pascanu, Volodymyr Mnih, Koray
  Kavukcuoglu, and Raia Hadsell.
\newblock {Policy Distillation}.
\newblock \emph{{ICLR}}, 2016.

\bibitem[Sennrich et~al.(2016)Sennrich, Haddow, and Birch]{sennrich-acl-2016}
Rico Sennrich, Barry Haddow, and Alexandra Birch.
\newblock {Neural Machine Translation of Rare Words with Subword Units}.
\newblock \emph{{ACL}}, 2016.

\bibitem[Serdyuk et~al.(2018)Serdyuk, Ke, Sordoni, Trischler, Pal, and
  Bengio]{serdyuk-iclr-2018}
Dmitriy Serdyuk, Nan~Rosemary Ke, Alessandro Sordoni, Adam Trischler, Chris
  Pal, and Yoshua Bengio.
\newblock {Twin Networks: Matching the Future for Sequence Generation}.
\newblock \emph{{ICLR}}, 2018.

\bibitem[Sriram et~al.(2018)Sriram, Jun, Satheesh, and Coates]{sriram2017cold}
Anuroop Sriram, Heewoo Jun, Sanjeev Satheesh, and Adam Coates.
\newblock {Cold fusion: Training seq2seq models together with language models}.
\newblock In \emph{{INTERSPEECH}}, 2018.

\bibitem[Sun et~al.(2017)Sun, Venkatraman, Gordon, Boots, and
  Bagnell]{sun2017deeply}
Wen Sun, Arun Venkatraman, Geoffrey~J Gordon, Byron Boots, and J~Andrew
  Bagnell.
\newblock Deeply aggrevated: Differentiable imitation learning for sequential
  prediction.
\newblock \emph{arXiv preprint arXiv:1703.01030}, 2017.

\bibitem[Sun et~al.(2018)Sun, Bagnell, and Boots]{sun2018truncated}
Wen Sun, J~Andrew Bagnell, and Byron Boots.
\newblock Truncated horizon policy search: Combining reinforcement learning \&
  imitation learning.
\newblock \emph{arXiv preprint arXiv:1805.11240}, 2018.

\bibitem[Sutskever et~al.(2014)Sutskever, Vinyals, and Le]{sutskeveretal14}
Ilya Sutskever, Oriol Vinyals, and Quoc~V. Le.
\newblock {Sequence to Sequence Learning with Neural Networks}.
\newblock \emph{NIPS}, 2014.

\bibitem[Sutton and Barto(1998)]{suttonbarto98}
Richard~S. Sutton and Andrew~G. Barto.
\newblock \emph{Reinforcement Learning: An Introduction}.
\newblock MIT Press, 1998.

\bibitem[Tjandra et~al.(2018)Tjandra, Sakti, and Nakamura]{tjandra-icassp-2018}
Andros Tjandra, Sakriani Sakti, and Satoshi Nakamura.
\newblock {Sequence-to-Sequence ASR Optimization via Reinforcement Learning}.
\newblock \emph{{ICASSP}}, 2018.

\bibitem[Vaswani et~al.(2017)Vaswani, Shazeer, Parmar, Uszkoreit, Jones, Gomez,
  Kaiser, and Polosukhin]{transformer2017}
Ashish Vaswani, Noam Shazeer, Niki Parmar, Jakob Uszkoreit, Llion Jones,
  Aidan~N Gomez, {\L}ukasz Kaiser, and Illia Polosukhin.
\newblock Attention is all you need.
\newblock \emph{NIPS}, 2017.

\bibitem[Vinyals et~al.(2015)Vinyals, Kaiser, Koo, Petrov, Sutskever, and
  Hinton]{vinyals-nips-2015}
Oriol Vinyals, Lukasz Kaiser, Terry Koo, Slav Petrov, Ilya Sutskever, and
  Geoffrey Hinton.
\newblock {Grammar as a Foreign Language}.
\newblock \emph{{NIPS}}, 2015.

\bibitem[Wang et~al.(2018)Wang, Pham, Dai, and Neubig]{wang2018switchout}
Xinyi Wang, Hieu Pham, Zihang Dai, and Graham Neubig.
\newblock Switchout: an efficient data augmentation algorithm for neural
  machine translation.
\newblock \emph{EMNLP}, 2018.

\bibitem[Wiseman and Rush(2016)]{wiseman-emnlp-2016}
Sam Wiseman and Alexander~M. Rush.
\newblock {Sequence-to-Sequence Learning as Beam-Search Optimization}.
\newblock \emph{EMNLP}, 2016.

\bibitem[Wu et~al.(2016)Wu, Schuster, Chen, Le, Norouzi, Macherey, Krikun, Cao,
  Gao, Macherey, et~al.]{wu2016google}
Yonghui Wu, Mike Schuster, Zhifeng Chen, Quoc~V Le, Mohammad Norouzi, Wolfgang
  Macherey, Maxim Krikun, Yuan Cao, Qin Gao, Klaus Macherey, et~al.
\newblock {Google's neural machine translation system: Bridging the gap between
  human and machine translation}.
\newblock \emph{arXiv:1609.08144}, 2016.

\bibitem[Xu et~al.(2015)Xu, Ba, Kiros, Cho, Courville, Salakhutdinov, Zemel,
  and Bengio]{showattendtell2015}
Kelvin Xu, Jimmy Ba, Ryan Kiros, Kyunghyun Cho, Aaron Courville, Ruslan
  Salakhutdinov, Richard Zemel, and Yoshua Bengio.
\newblock {Show, Attend and Tell: Neural Image Caption Generation with Visual
  Attention}.
\newblock \emph{ICML}, 2015.

\bibitem[Zeyer et~al.(2018)Zeyer, Irie, Schl{\"u}ter, and
  Ney]{zeyer2018improved}
Albert Zeyer, Kazuki Irie, Ralf Schl{\"u}ter, and Hermann Ney.
\newblock Improved training of end-to-end attention models for speech
  recognition.
\newblock \emph{arXiv preprint arXiv:1805.03294}, 2018.

\bibitem[Zhang et~al.(2017)Zhang, Chan, and Jaitly]{zhang-icassp-2017}
Yu~Zhang, William Chan, and Navdeep Jaitly.
\newblock {Very Deep Convolutional Networks for End-to-End Speech Recognition}.
\newblock \emph{{ICASSP}}, 2017.

\bibitem[Zhong et~al.(2017)Zhong, Xiong, and Socher]{seq2sql}
Victor Zhong, Caiming Xiong, and Richard Socher.
\newblock {Seq2SQL: Generating Structured Queries from Natural Language using
  Reinforcement Learning}.
\newblock \emph{arXiv:1709.00103}, 2017.

\end{thebibliography}
